\documentclass[journal,compsoc]{IEEEtran}
\usepackage[breaklinks=true,colorlinks,citecolor=black,linkcolor=black,urlcolor=black]{hyperref}
\usepackage[dvipsnames,table,xcdraw]{xcolor}

\usepackage[nocompress]{cite}
\usepackage{graphicx}
\usepackage{amsmath}
\usepackage{amssymb}
\usepackage{amsthm}
\usepackage{dsfont}
\usepackage{multirow}
\usepackage{bm}
\usepackage[ruled]{algorithm}
\usepackage{algpseudocode}
\usepackage{overpic}
\usepackage{pifont}
\usepackage[caption=false]{subfig}

\usepackage{silence}
\newcommand{\hlt}[1]{\textbf{{#1}}} 
\newtheorem{definition}{Definition}
\newtheorem{proposition}{Proposition}
\newtheorem{corollary}{Corollary}
\newtheorem{lemma}{Lemma}
\newtheorem{theorem}{Proposition}

\newcommand{\minor}[1]{\textcolor{black}{#1}}

\newlength\savewidth\newcommand\shline{\noalign{\global\savewidth\arrayrulewidth
  \global\arrayrulewidth 1pt}\hline\noalign{\global\arrayrulewidth\savewidth}}

\newcommand{\tablestyle}[2]{\setlength{\tabcolsep}{#1}\renewcommand{\arraystretch}{#2}\centering}

\makeatletter\renewcommand\paragraph{\@startsection{paragraph}{4}{\z@}
  {.5em \@plus1ex \@minus.2ex}{-.5em}{\normalfont\normalsize\bfseries}}\makeatother
  
\newcommand{\myPara}[1]{\vspace{.05in}\noindent\textbf{#1}\quad}

\def\vs{\emph{vs.}}
\def\ie{\emph{i.e.,~}}
\def\eg{\emph{e.g.,~}}
\def\etc{\emph{etc}}
\def\wrt{\emph{w.r.t.}~}
\def\etal{{\em et al.}}
\def\aka{{\em a.k.a.}~}
\def\chmk{\checkmark}
\newcommand{\figref}[1]{Fig.~\ref{#1}}%
\newcommand{\tabref}[1]{Tab.~\ref{#1}}%
\newcommand{\secref}[1]{Sec.~\ref{#1}}
\renewcommand{\eqref}[1]{Eq.~(\ref{#1})}
\newcommand{\propref}[1]{Appendix~(A.\ref{#1})}
\newcommand{\coloref}[1]{Appendix~(A\ref{#1})}

\begin{document}

\title{Localization Distillation for Object Detection}

\author{Zhaohui Zheng,
        Rongguang Ye,
        Qibin Hou,~\IEEEmembership{Member,~IEEE},
        Dongwei Ren,~\IEEEmembership{Member,~IEEE},
        Ping Wang,\\
        Wangmeng Zuo,~\IEEEmembership{Senior Member,~IEEE},
        and Ming-Ming Cheng,~\IEEEmembership{Senior Member,~IEEE},
\IEEEcompsocitemizethanks{
\IEEEcompsocthanksitem Z. Zheng, Q. Hou and M.M. Cheng are with TMCC, 
  CS, Nankai University, Tianjin, China. 
  E-mail: Zh\_zheng@mail.nankai.edu.cn; \{houqb,cmm\}@nankai.edu.cn
\IEEEcompsocthanksitem R. Ye and P. Wang are with the School of Mathematics, 
  Tianjin University, China. E-mail: \{ementon,wang\_ping\}@tju.edu.cn
\IEEEcompsocthanksitem D. Ren and W. Zuo are with the School of 
  Computer Science and Technology, Harbin Institute of Technology, China. 
  E-mail: \{rendongweihit,cswmzuo\}@gmail.com
\IEEEcompsocthanksitem Q. Hou is the corresponding author.
}

\thanks{Manuscript received Apr. 12, 2022; revised Dec. 4, 2022.}}
%
%

\markboth{Journal of \LaTeX\ Class Files,~Vol.~14, No.~8, August~2015}%
{Shell \MakeLowercase{\textit{et al.}}: Bare Demo of IEEEtran.cls for Computer Society Journals}

\IEEEtitleabstractindextext{%
\begin{abstract}
%
Previous knowledge distillation (KD) methods for object detection mostly 
focus on feature imitation instead of mimicking the prediction logits
due to its inefficiency in distilling the localization information.
In this paper, 
we investigate whether logit mimicking always lags behind feature imitation.
Towards this goal, we first present a novel localization distillation (LD) 
method which can efficiently transfer the localization knowledge 
from the teacher to the student.
Second, we introduce the concept of valuable localization region
that can aid to selectively distill the classification and 
localization knowledge for a certain region.
Combining these two new components, for the first time, 
we show that logit mimicking can outperform feature imitation 
and the absence of localization distillation is a critical reason for 
why logit mimicking under-performs for years.
The thorough studies exhibit the great potential of logit mimicking 
that can significantly alleviate the localization ambiguity, 
learn robust feature representation, 
and ease the training difficulty in the early stage.
We also provide the theoretical connection between the proposed LD and 
the classification KD, 
that they share the equivalent optimization effect.
Our distillation scheme is simple as well as effective and can 
be easily applied to both dense horizontal object detectors and 
rotated object detectors.
Extensive experiments on the MS COCO, PASCAL VOC, and DOTA benchmarks 
demonstrate that our method can achieve considerable AP improvement without 
any sacrifice on the inference speed.
Our source code and pretrained models are publicly available at 
\url{https://github.com/HikariTJU/LD}.
\end{abstract}

}

\maketitle

\IEEEdisplaynontitleabstractindextext

%
\IEEEpeerreviewmaketitle

\IEEEraisesectionheading{\section{Introduction}\label{sec:introduction}}

\IEEEPARstart{A}{s} a model compression technology, 
knowledge distillation (KD) \cite{hinton2015distilling,FitNets} 
has been an efficient technique in 
learning compact models to mitigate the computational burden.
It has been widely validated to be useful for boosting the performance 
of small-sized student networks 
by transferring the generalized knowledge captured by 
large-sized teacher networks \cite{hinton2015distilling,FitNets,
zagoruyko2016paying,kim2018paraphrasing,Jin_2019_ICCV,wang2021distilling}.
Speaking of KD in object detection, 
there are mainly three popular KD pipelines as shown in \figref{fig:previous}.
Logit mimicking \cite{hinton2015distilling}, also known as classification KD, 
is originally designed for image classification, 
where the KD process operates on the logits of the teacher-student pair.
Feature imitation, motivated by the pioneer work FitNet~\cite{FitNets}, 
aims to enforce the consistency of the feature representations 
between the teacher-student pair.
The last one, namely the pseudo bounding box regression, 
uses the predicted bounding boxes from the teacher as an addition supervision
to the bounding box prediction branch of the student.

\begin{figure}[!t]
	\centering
	\includegraphics[width=\linewidth]{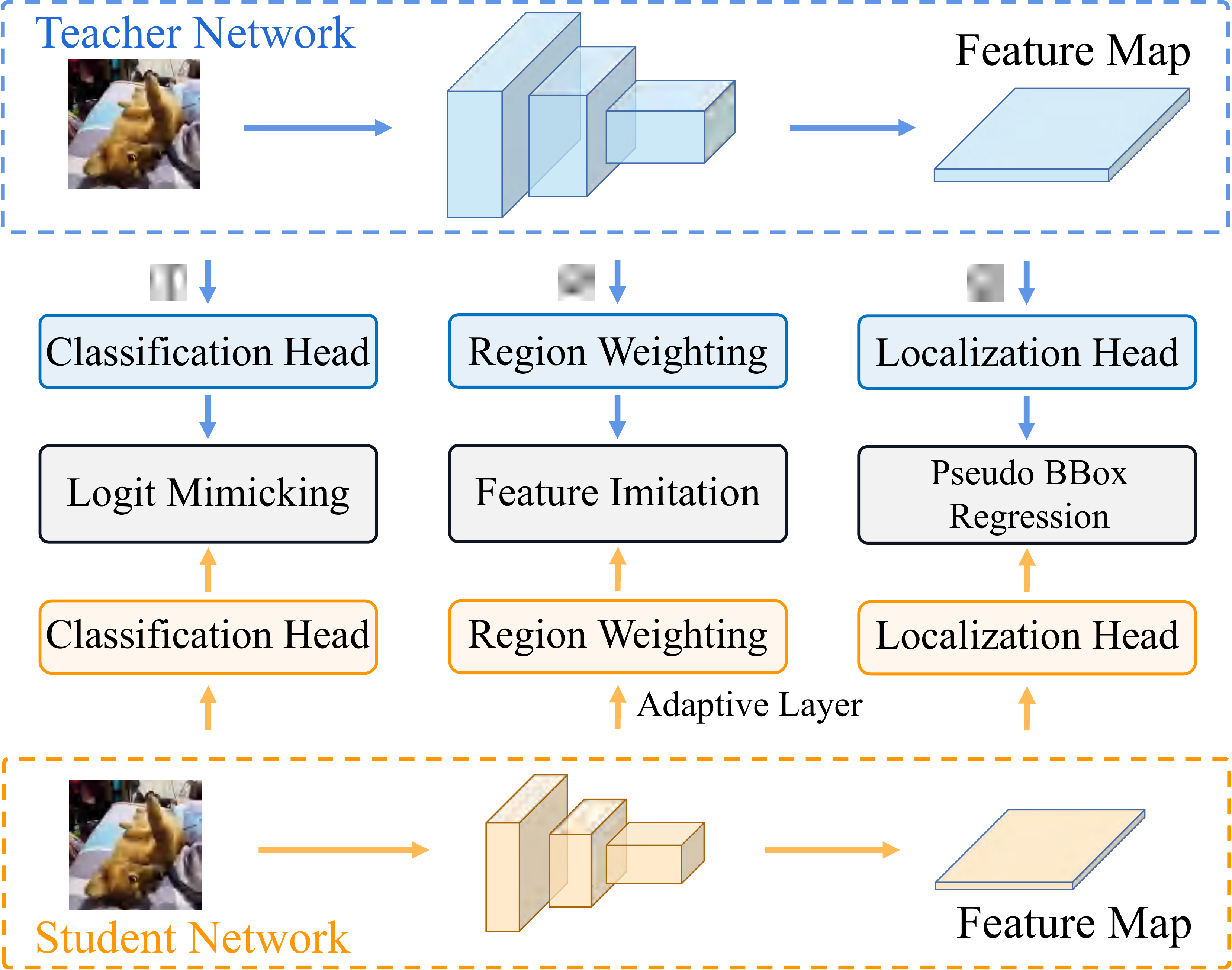} \\ 
	\caption{Existing KD pipelines for object detection. 
    \ding{172} Logit Mimicking: 
    classification KD in \cite{hinton2015distilling}. 
    \ding{173} Feature Imitation: recent popular methods distill intermediate 
    features based on various distillation regions, 
    which usually need adaptive layers to align 
    the size of the student's feature map. 
    \ding{174} Pseudo BBox Regression: treating teachers' 
    predicted bounding boxes as additional regression targets
    \cite{chen2017learning,sun2020distilling}.
	}\label{fig:previous}
\end{figure}

Among these methods, 
the original logit mimicking technique~\cite{hinton2015distilling}
for classification is often inefficient 
as it only transfers the classification knowledge 
while neglects the importance of localization knowledge distillation.
Therefore, existent KD methods for object detection mostly focus on 
feature imitation, 
and demonstrate that distilling the feature representations is 
more advantageous than distilling the logits
\cite{wang2019distilling,zhang2020improve,kang2021instanceconditional}.
We summarize three crucial reasons for this phenomenon:
First of all, the effectiveness of logit mimicking partially relies 
on the number of classes which may vary in different application scenarios
\cite{wang2019distilling}.
Second, the logit mimicking can only be applied to the classification head, 
which cannot distill the localization information.
At last, in the framework of multi-task learning, feature imitation can 
transfer the hybrid knowledge of classification and localization 
which can benefit the downstream classification and localization tasks.

\begin{figure}[!t]
	\centering
	\includegraphics[height=0.433\linewidth]{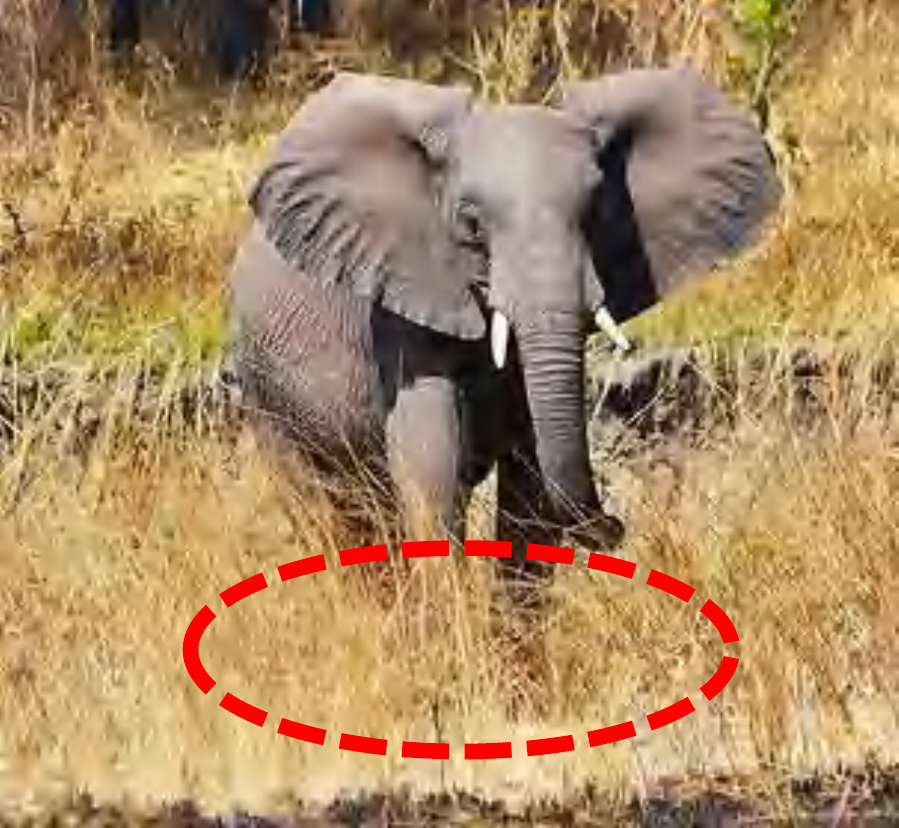} \hfill
	\includegraphics[height=0.433\linewidth]{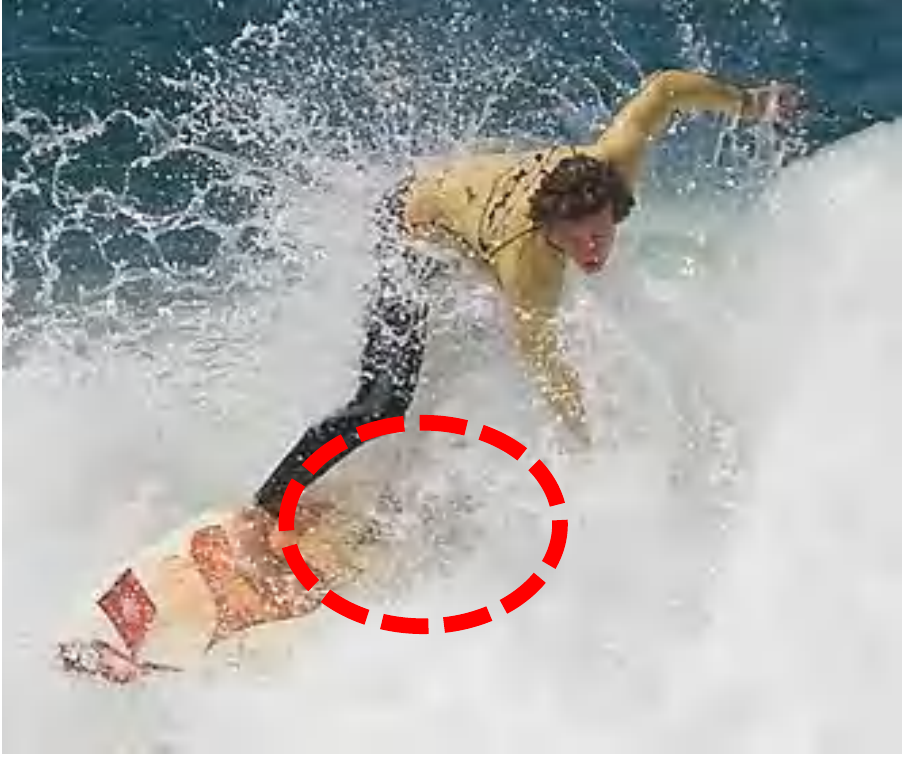}\\
	\caption{Bottom edge for ``\emph{elephant}" and right edge for 
    ``\emph{surfboard}" are ambiguous to locate.
	}\label{fig:ambiguity}
\end{figure}

In this work, 
we examine the aforementioned common belief in object detection KD, 
and challenge whether feature imitation always stays ahead of logit mimicking?
For this purpose, 
we firstly present a simple yet effective localization distillation (LD) method 
which is inspired by an interesting observation 
that the bounding box distributions generated by the teacher
\cite{gfocal,offsetbin} 
can serve as a strong supervision to the student detector.
The bounding box distribution \cite{gfocal,offsetbin} 
is originally designed to model the real distributions of bounding boxes, 
an efficient way to solve the localization ambiguity as shown in 
\figref{fig:ambiguity}.
With the discretized probability distribution representations, 
the localizer can reflect the localization ambiguity by the flatness and 
sharpness of the distribution, 
which is not held in the conventional Dirac delta representation 
of bounding boxes~\cite{yolov1,SSD,fasterrcnn,FCOS}.
This allows our LD to efficiently transfer richer localization knowledge 
from the teacher to the student than using pseudo bounding box regression 
(right part in \figref{fig:previous}).

Combining the proposed LD and the classification KD yields a unified KD method
based on a pure logit mimicking framework for both the classification branch
and the localization branch.
As logit mimicking enables us to separately distill the classification 
knowledge and the localization one, 
we found that these two sub-tasks favor different distillation regions.
Motivated by this, 
we introduce the concept of valuable localization region (VLR) and propose 
to conduct distillation in a \minor{selective region distillation} manner.
We will show the advantage of using VLR in our distillation framework in
the experiment section.

Furthermore, 
we comprehensively discuss the technical details of LD and elaborate
on the behavior of logit mimicking and feature imitation.
Intriguingly, we observe that logit mimicking can outperform feature imitation 
for the first time, 
which indicates that the absence of localization distillation is 
actually the key reason 
why logit mimicking under-performs in object detection for years.
Another observation is that we find the reason 
why logit mimicking works is not because of 
the consistency of the feature representations between the teacher-student pair.
Conversely, the student learns significantly different feature representations 
from the teacher's in terms of the $l_n$ distance and linear correlation.
We also observe that if the student is trained with feature imitation, 
it tends to produce a sharp AP score landscape in the feature subspace, 
and aggravates the training difficulty in the early training stage.

The above observations reflect the great potentials of logit mimicking 
over feature imitation: 
1) being able to separately transfer different types of knowledge, 
2) learning more robust feature representations, and 
3) easing the training difficulty.
Our method is simple and can be easily equipped with in both horizontal 
and rotated object detectors to improve their performance 
without introducing any inference overhead.
Extensive experiments on MS COCO show that without bells and whistles, 
we can lift the AP score of the strong baseline GFocal~\cite{gfocal} 
with ResNet-50-FPN backbone from 40.1 to 42.1, and AP$_{75}$ from 43.1 to 45.6.
Our best model using ResNeXt-101-32x4d-DCN backbone can achieve 
a single-scale test of 50.5 AP, 
which surpasses all existing detectors under the same backbone, neck, 
and test settings.
PyTorch \cite{PyTorch} and Jittor \cite{hu2020jittor} version of 
the source code and pretrained models are publicly available at 
\url{https://github.com/HikariTJU/LD}.

The main contributions of this paper are four-fold:
\begin{enumerate}
  \item We present a novel localization distillation method that 
    greatly improve the distillation efficiency of logit mimicking 
    in object detection.
  \item We provide exploratory experiments and analysis for the 
    behavior of logit mimicking and feature imitation. 
    To our best knowledge, this is the first work revealing the 
    great potential of logit mimicking over feature imitation.
  \item We present a \minor{selective region distillation} based on 
    the newly introduced valuable localization region 
    to better distill the student detector.
  \item We extend our LD to a rotated version so that it can be applied 
    to arbitrary-oriented object detection.
\end{enumerate}

This paper is a substantial extension of its previous conference version
\cite{zheng2022LD}.
In particular, (a) We provide theoretical connection for the proposed LD 
and the classification KD that they share the equivalent optimization effects;
%
(b) We conduct more detailed and insightful analysis for logit mimicking 
and feature imitation, 
including the different characteristics of the learned feature representations 
and logits, and the training difficulty of feature imitation;
(c) We extend the original LD to a more generic version, namely rotated LD, 
which can distill arbitrary-oriented object detectors.

\section{Related Work}\label{sec:related}
	

\subsection{Knowledge Distillation}

Knowledge distillation \cite{hinton2015distilling,
Zagoruyko2017AT,bae2020densely,relationKD,TA,DenselyTA},
as a hot research topic, has been deeply studied recently.
The fundamental idea is to use a well-performed large-sized teacher network 
to transfer the captured knowledge to the small-sized student network.
Logit mimicking, \aka classification KD, was first introduced by 
Hinton \etal \cite{hinton2015distilling}, 
where the logit outputs of the student classifier are supervised by 
those of the teacher classifier.
Later, FitNet \cite{FitNets} extends the teacher-student learning framework by 
mimicking the intermediate-level hints from the hidden layers 
of the teacher model.
Knowledge distillation was first applied to object detection in 
\cite{chen2017learning}, where the hint learning, classification KD, 
and pseudo bounding box regression were simultaneously used for 
multi-class object detection.
However, an object detector requires not only precise classification ability, 
but also strong localization ability.
The absence of localization knowledge distillation limits the performance of 
the conventional KD method.

To tackle the above issue, many feature imitation methods have been developed, 
most of which focus on where to distill and loss function weighting.
Among these, Li \etal \cite{Li_2017_CVPR} proposed to mimic the features
within the region proposal for Faster R-CNN.
Wang \etal \cite{wang2019distilling} imitated the fine-grained features 
on close anchor box locations.
Recently, Dai \etal \cite{GIbox} introduced the 
General Instance Selection Module to mimic deep features 
within the discriminative patches between teacher-student pairs.
DeFeat \cite{defeat} leverages different loss weights when 
conducting feature imitation on the object regions and the background region.
%
%
%
There are also various feature imitation methods from 
the perspective of weighted imitation loss, 
including Gaussian mask weighted~\cite{sun2020distilling}, 
feature richness weighted~\cite{FeatureRichness}, 
and prediction-guided imitation loss~\cite{li2021knowledge}.
Unlike the aforementioned methods, 
our work introduces localization distillation and demonstrate that 
logit mimicking can outperform feature imitation for KD in object detection.

\subsection{Object Localization}
Object localization is a fundamental issue in object detection 
\cite{locnet,wang2019region,SABL,zhu2019feature,gridrcnn,kong2018deep,SCRDet,GWD,KLD,VFNet}.
Bounding box regression is the most popular way so far for localization 
in object detection \cite{felzenszwalb2009object,yolov1,SSD,fasterrcnn,Han2019221}, 
where the Dirac delta distribution representation has been used for years.
R-CNN series \cite{fasterrcnn,cascadercnn,librarcnn,DynamicRCNN} 
adopt multiple regression stages to refine the detection results, 
while YOLO series\cite{yolov1,yolov2,yolov3,yolov4}, 
SSD series \cite{SSD,DSSD,STDN}, 
and FCOS series \cite{FCOS,gfocal} adopt one-stage regression.
In \cite{unitbox,giou,diou,ciou}, IoU-based loss functions are proposed 
to improve the localization quality of bounding boxes.
Recently, bounding box representation has evolved from 
Dirac delta distribution \cite{yolov1,SSD,fasterrcnn} to 
Gaussian distribution \cite{softernms,gaussian_yolov3}, 
and further to probability distribution \cite{offsetbin,gfocal}.
The probability distribution of bounding boxes is more comprehensive 
for describing the uncertainty of bounding boxes, and has been validated to 
be the most advanced bounding box representation so far.

\subsection{Localization Quality Estimation}

As the name suggests, Localization Quality Estimation (LQE) predicts a score 
that measures the localization quality of the bounding boxes 
predicted by the detector.
LQE is usually used to cooperate with the classification task 
during training \cite{gfocalv2}, 
\ie enhancing the consistency between classification and localization.
It can also be applied in joint decision-making during post-processing 
\cite{yolov1,iounet,FCOS}, 
which considers both the classification score and LQE when performing NMS.
Early research can be dated to YOLOv1 \cite{yolov1}, 
where the predicted object confidence is used to penalize 
the classification score.
Then, box/mask IoU \cite{iounet,mask_scoring} and 
box/polar centerness \cite{FCOS,polarmask} are proposed 
to model the uncertainty of detections in object detection and 
instance segmentation, respectively.
For bounding box representation, 
Softer-NMS \cite{softernms} and Gaussian YOLOv3 \cite{gaussian_yolov3} 
predict variances for each edge of the bounding boxes.
LQE is a preliminary approach to model localization ambiguity.

\subsection{Arbitrary-Oriented Object Detection}

Driven by the success of object detection, 
rotated object detection has become a hot topic in computer vision recently 
\cite{zhou2022mmrotate}.
%
The mainstream rotated object detectors, such as RRPN~\cite{RRPN}, 
generate rotated proposals based on Faster R-CNN \cite{fasterrcnn}, 
while Rotated-RetinaNet~\cite{lin2017focal} directly predicts an additional 
rotated angle based on RetinaNet.
To address the boundary discontinuity and square-like problems, 
SCRDet \cite{SCRDet} and RSDet \cite{RSDet} propose IoU-smooth L1 loss 
and modulated loss respectively for attaining smoother boundary loss, 
and CSL~\cite{CSL} proposes to use angle classification 
instead of angle regression.

Different from the horizontal bounding box regression 
which can easily leverage the IoU-based losses 
(\eg GIoU \cite{giou}, DIoU \cite{diou}, and CIoU \cite{ciou}) 
to enhance localization ability, 
the Skew IoU loss for rotated bounding box regression is quite difficult 
to implement due to the complexity of the backward propagation 
in existing deep learning libraries~\cite{tensorflow,PyTorch,hu2020jittor}.
PIoU loss \cite{piou} approximates the Skew IoU by accumulating the pixels 
of the intersection and union of two rotated bounding boxes.
GWD \cite{GWD} and KLD~\cite{KLD} model the rotated bounding boxes 
via the 2D Gaussian Distribution representation and 
propose to use the Gaussian Wasserstein distance and KL divergence 
to simulate the Skew IoU loss, respectively.
More recently, based on the 2D Gaussian distribution representation 
of rotated bounding boxes, 
Yang \etal~\cite{kfiou} proposed the KFIoU loss by exploiting the 
Kalman filter formulation to mimic the Skew IoU in the trend level.
To sum up, the rotated regression-based detectors are still dominating 
this task owing to their simplicity and strong performance.

\begin{figure*}[!t]
  \centering
  \begin{overpic}[width=\textwidth]{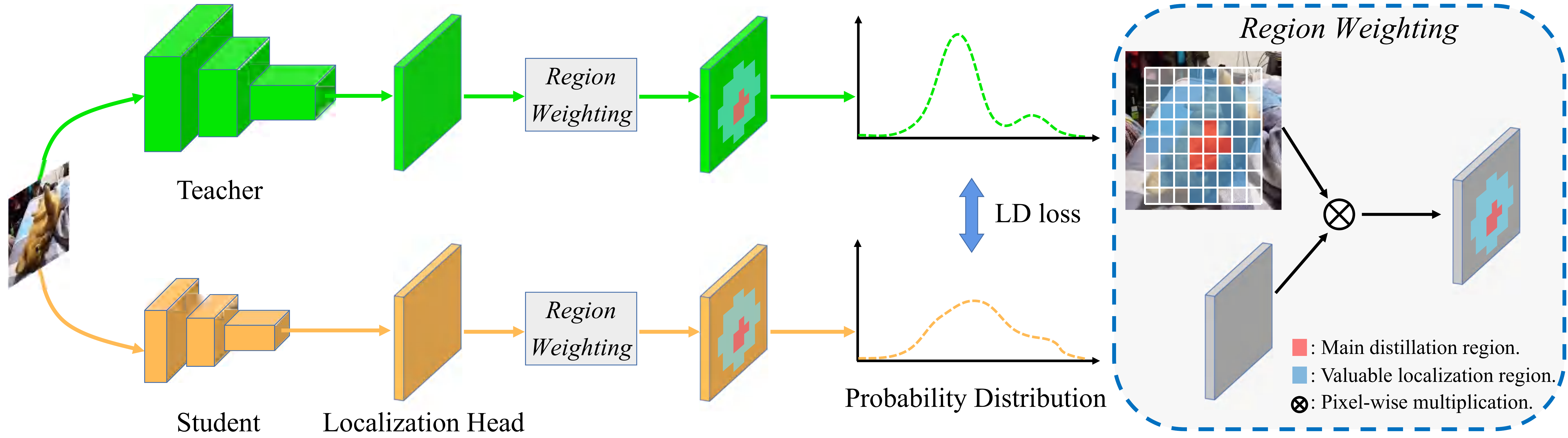}
    \put(49.3,8){$\mathcal{S}(\cdot,\tau)$}
    \put(49.3,23){$\mathcal{S}(\cdot,\tau)$}
  \end{overpic}\\ \vspace{-5pt}
  \caption{Illustration of localization distillation (LD) 
    for an edge $e \in \mathcal{B}$. 
    Only the localization branch is visualized here. 
    $\mathcal{S}(\cdot,\tau)$ is the generalized SoftMax function 
    with temperature $\tau$. 
    For a given detector, we first switch the bounding box representation 
    to probability distribution. 
    Then, we determine where to distill via region weighting 
    on the main distillation region and the valuable localization region. 
    Finally, we calculate the LD loss between two probability distributions 
    predicted by the teacher and the student.
  }\label{fig:LD}
\end{figure*}

\section{Approach} \label{approach}

To begin with, we revisit the knowledge distillation background, 
including logit mimicking and feature imitation.
Next, we describe our simple yet effective localization distillation (LD)
and explain how to apply LD for rotated object detection.
Then, we analyze the property of the proposed LD loss, 
especially the theoretical connection to the classification KD.
In addition, we also introduces the concept of valuable localization region 
for better distilling the localization knowledge in our framework.
Finally, we describe the \minor{selective region distillation} 
based on the newly introduced valuable localization region 
and give the optimization objective.

\subsection{Preliminaries}\label{sec:background}

\minor{In the KD pipeline of object detection, 
the input image is fed into two object detectors, 
\ie the student detector and the frozen teacher detector.
The distillation process forces the outputs of the student to mimic 
those of the teacher.
There are two mainstream paradigms of KD methods in object detection.}

\newcommand{\LossT}[1]{\mathcal{L}_\text{#1}}
\newcommand{\mcH}[2]{\mathcal{H}(#1,#2)}
\newcommand{\mcHbm}[2]{\mathcal{H}(\bm{#1},\bm{#2})}
\newcommand{\mcHt}[2]{\mathcal{H}(\bm{#1}_\tau,\bm{#2}_\tau)}
\newcommand{\mcS}[2]{\mathcal{S}(\bm{z}_{#1}, #2)}

\myPara{Logit mimicking.}
\minor{The logit mimicking (LM) is first developed for image classification 
\cite{hinton2015distilling}, 
in which the student model can be improved by mimicking the soft output 
of the teacher classifier.
Let $\bm{z}_S,\bm{z}_T\in\mathbb{R}^{W\times H\times C}$
be the logits predicted by the student and the teacher, respectively.
$W$ and $H$ represent the output size of the logit maps.
$C$ denotes the number of classes.}
These logits are then transformed into probability distributions 
$\bm{p}_\tau$ and $\bm{q}_\tau$ by using the generalized SoftMax function.
%
We can train the network by minimizing the loss:
\begin{align}\label{eq:kd}
  \mathcal{L} 
  &= \LossT{CE} + \lambda \LossT{KD} \\
  &= \mcHbm{p}{g} + \lambda \mcHt{p}{q},
\end{align}
where $\bm{p}$ is the predicted probability vectors, 
$\bm{g}=\{0,1\}^n$ is the one-hot ground-truth label, 
$\mathcal{H}$ is the cross-entropy loss, 
and $\lambda$ balances the two loss terms.
\minor{For object detection, the distillation can be carried out on 
some pre-defined distillation region $\mathcal{R}$.}

\myPara{Feature imitation.} 
Recently, it has been found that feature imitation (FI), 
which aims to transfer knowledge by imitating the deep features 
between teacher-student pairs, 
works better than the classification KD \cite{FitNets,wang2019distilling}.
Mathematically, the feature imitation procedure can be formulated as:
\begin{equation}\label{eq:im}
  \LossT{FI} = \frac{1}{|\mathcal{R}|}\sum_{r \in \mathcal{R}}
  ||\tilde{M}_S(r)-M_T(r)||_2,
\end{equation}
where $\mathcal{R}$ is the imitation region, 
and $|\cdot|$ is the cardinality of the region.
Note that an adaptive layer is needed to transform the size of student's 
feature map $M_S$ to be the same as the teacher's $M_T$,
so that $\tilde{M}_S,M_T\in\mathbb{R}^{W\times H\times D}$.

\myPara{Bounding box representation.}
\minor{For a given bounding box $\mathcal{B}$, 
the conventional representations have two forms, \ie 
$\{\delta_x,\delta_y,\delta_w,\delta_h\}$ 
(encoding the coordinate mappings of the central point, 
the width and the height from the anchor box to the ground-truth box)
\cite{yolov1,SSD,fasterrcnn} 
and $\{t,b,l,r\}$ 
(the distances from the sampled point to the top, bottom, left, and right edges)
\cite{FCOS}.
These two forms actually follow the Dirac delta distribution that 
only focuses on the ground-truth locations but cannot model 
the ambiguity of bounding boxes as shown in \figref{fig:ambiguity}.
This is also clearly demonstrated in some previous works
\cite{softernms,gfocal}.}

\subsection{Localization Distillation} \label{sec:LD}

\minor{In this subsection, we present localization distillation (LD), 
a new way to enhance the distillation efficiency for object detection.
Our LD is evolved from the view of probability distribution representation 
of bounding boxes 
(anchor free \cite{gfocal} and anchor-based \cite{offsetbin}), 
which is originally designed for generic object detection and 
carries abundant localization information.
The working principle of our LD can be seen in \figref{fig:LD}.
The procedure is the same to both anchor-based and anchor-free detectors.}

\minor{Given an object detector, we follow \cite{gfocal,offsetbin} to convert the bounding box representation 
from a quaternary representation to a probability distribution.
Let $e\in\mathcal{B}$ 
be one of the regression variables of bounding box, 
whose regression range is $[e_{\min}, e_{\max}]$.
The bounding box distribution quantizes the continuous regression range 
into a uniform discretized variable 
$\bm{e}=[e_0,e_1,\cdots,e_n] \in \mathbb{R}^{n+1}$
with $n$ sub-intervals, where $e_0=e_{\min}$ and $e_n=e_{\max}$.
The localization head predicts $n+1$ logits 
$\bm{z}=\{z_0,z_1,\cdots,z_n\}$, 
corresponding to the endpoints of the subintervals $\{e_0,e_1,\cdots,e_n\}$.
Each edge of the given bounding box can be represented as a probability 
distribution by using the SoftMax function. 
For the number of the subinterval $n$, 
we follow the settings of GFocal \cite{gfocal}, 
and a recommended choice of $n$ is $8\sim16$.}
Different from \cite{offsetbin,gfocal}, 
we transform $\bm{z}_S$ and $\bm{z}_T$ into the probability distributions 
$\bm{p}_\tau$ and $\bm{q}_\tau$ using the generalized 
SoftMax function $\mathcal{S}(\cdot,\tau)$.
Note that when $\tau=1$, it is equivalent to the original SoftMax function.
When $\tau\rightarrow0$, it tends to be a Dirac delta distribution.
When $\tau\rightarrow\infty$, it will be a uniform distribution.
Empirically, $\tau > 1$ is set to soften the distribution, 
making the bounding box distribution carry more information.
The localization distillation for measuring the similarity 
between the two probability vectors $\bm{p}_\tau, \bm{q}_\tau\in\mathbb{R}^n$
for one of the bounding box representation $\bm{e}$ is attained by:
\begin{align}
\LossT{LD}^{\bm e} 
& = \mathcal{H}(\bm{p}_\tau,\bm{q}_\tau) \\
& = \mathcal{H}(\mcS{S}{\tau}, \mcS{T}{\tau}).
\end{align}
Then, LD for all the four edges of some bounding box $\mathcal{B}$ 
can be formulated as:
\begin{equation}\label{eq:ld}
\LossT{LD} (\mathcal{B}_S, \mathcal{B}_T)
= \sum_{\bm{e} \in \mathcal{B}} \LossT{LD}^{\bm{e}},
\end{equation}

where $\mathcal{B}_S, \mathcal{B}_T$ 
are the predicted bounding boxes of the student and the teacher, respectively.

\subsection{Rotated LD}\label{sec:rotatedLD}

Our LD can also be flexibly used to distill rotated bounding box detectors.
Parametric regression is the most popular manner in the classical dense 
regression-based rotated object detection \cite{SCRDet,r3det,GWD,KLD}.
\minor{$\mathcal{B}=\{\delta_x,\delta_y,\delta_w,\delta_h,\delta_\theta\}$} 
is commonly used to represent a rotated bounding box, 
where $\delta_\theta$ denotes the encoded rotated angle.
%
%
To conduct rotated localization distillation, we firstly 
generate the lower and upper bounds of the regression range 
$[e_{\min},e_{\max}]$, 
where $e\in\mathcal{B}$.

Note that the rotated angle prediction $\delta_\theta$ usually has 
a different regression range from $\delta_x,\delta_y,\delta_w,\delta_h$.
Thus, different lower and upper bounds of regression ranges are set for them.
\minor{In practice, $[e_{min},e_{max}]\subset[-5,5]$ will be an acceptable choice.}
%
%
Then, we convert the rotated bounding box to rotated bounding box distributions, as \secref{sec:LD} describes.
Finally, the LD loss is calculated according to \eqref{eq:ld}
for the rotated bounding box distributions.

\subsection{Property of LD}\label{sec:property}
We can see that our LD holds the formulation of the standard logit mimicking.
The question one may ask is: Does LD also inherit the property of the classification KD, especially for the optimization process?
Different from the classification task where a unique integer
is treated as the ground-truth label, 
the ground-truth label of the localization task is a float point number $e^*$, 
whose value, for instance, is ranged in an interval $[e_i,e_{i+1}]$.
In the following, we show an important property of LD, demonstrating
that it can inherit the optimization effects held by the 
classification KD.

\begin{proposition}\label{pro1}
Let $\bm{s}$ be the student's predicted probability vector, 
and $u_1, u_2$ are two constants with $u_1 + u_2 = 1$. 
Then, we have:
\begin{enumerate}
  \item If $\bm{p},\bm{q}$ are two classification probability vectors, 
    LD effect on the linear combination $\bm{l}=u_1\bm{p}+u_2\bm{q}$ 
    is equal to the linear combination of KD effects on $\bm{p},\bm{q}$;
  \item If $\bm{l}$ is a localization probability vector, 
    LD effect on $\bm{l}$ is equal to two KD effects on its decomposition 
    $\bm{p}$ and $\bm{q}$.
\end{enumerate}
The above two share the same expression,
\begin{equation}\label{eq:KDLD}
  \partial LD^{\bm{l}}_i=u_1\partial KD^{\bm{p}}_i+u_2\partial KD^{\bm{q}}_i,
\end{equation}
where $\partial KD^{\bm{p}}_i$ denotes the derivatives of the KD loss of 
two probabilities $\bm{s},\bm{p}$ \wrt a given logit $z_i$, 
and $\partial LD^{\bm{p}}_i$ likewise for the LD loss.
\end{proposition}

The proof can be found in the \propref{appendix:1}.
Proposition \ref{pro1} provides 
the theoretical connection between LD and the classification KD. 
It shows that the optimization effects of LD on a float point number 
localization problem is functionally equivalent to two KD effects 
on the integer position classification problems.
%
%
Therefore, as a direct corollary of \cite{tang2020understanding}, 
LD holds the gradient rescaling to the distribution focal loss (DFL) 
\cite{gfocal} \wrt the relative prediction confidence at two near positions. 
For the details, we refer to the \coloref{appendix:2}.

\subsection{Valuable Localization Region}\label{sec:VLR}

Previous works mostly force the deep features of 
the student to mimic those of the teacher 
by minimizing the $l_2$ loss.
However, a straightforward question arises: 
Should we use the whole imitation regions without discrimination 
to distill the hybrid knowledge?
According to our observation, the answer is no.
In this subsection, we describe the valuable localization region (VLR) 
to further improve the distillation efficiency,
which we believe will be a promising way to train better student detectors.

Specifically, the distillation region is divided into two parts, 
the main distillation region and the valuable localization region.
The main distillation region is intuitively determined by label assignment, 
\ie the positive locations of the detection head.
The valuable localization region can be obtained by Algorithm~\ref{alg:VLR}.
First, we calculate the DIoU \cite{diou} matrix $\bm{X}$ between all the 
anchor boxes $\bm{B}^a$ and the ground-truth boxes $\bm{B}^{gt}$.
Then, we set the lower bound of DIoU to be $\alpha_{vl}=\gamma\alpha_{pos}$, 
where $\alpha_{pos}$ is the positive IoU threshold of label assignment.
The VLR can be defined as 
$\bm{V}=\{\alpha_{vl}\leqslant\bm{X}\leqslant\alpha_{pos}\}$.
Our method has only one hyperparameter $\gamma\leqslant1$, 
which controls the range of the VLRs.
When $\gamma=0$, all the locations whose DIoUs  
between the preset anchor boxes and the GT boxes satisfy 
$0\leqslant x_{ij}\leqslant\alpha_{pos}$ will be determined as VLRs.
When $\gamma\rightarrow 1$, the VLR will gradually shrink to empty.
Here we use DIoU \cite{diou} since it gives higher priority to the locations 
close to the center of the object.

\begin{algorithm}[!tb]	
  \small
  \caption{Valuable Localization Region}
  \label{alg:VLR}
\begin{algorithmic}[1]
  \Require{A set of anchor boxes $\bm{B}^a = \{\mathcal{B}^a_i\}$ and 
    a set of ground truth boxes $\bm{B}^{gt} = \{\mathcal{B}^{gt}_{j}\}$, $1\leqslant i\leqslant I$ , $1\leqslant j\leqslant J$. 
    Positive threshold $\alpha_{pos}$ of label assignment.
  }
  \Ensure{$\bm{V}=\{v_{ij}\}_{I\times J}, v_{ij}\in \{0,1\}$ 
    encodes final location of VLR, 
    where $1$ denotes VLR and $0$ indicates ignore.
  }
  \State Compute DIoU matrix $\bm{X}=\{x_{ij}\}_{I\times J}$ 
    with $x_{ij}=DIoU(\mathcal{B}^a_{i},\mathcal{B}^{gt}_j)$.
  \State $\alpha_{vl}=\gamma\alpha_{pos}$.
  \State Select locations with 
    $\bm{V}=\{\alpha_{vl}\leqslant\bm{X}\leqslant\alpha_{pos}\}$.
  \State\textbf{return} {$\bm{V}$}
\end{algorithmic}
\end{algorithm}

Similar to label assignment, 
our method assigns attributes to each location across multi-level FPN.
In this way, some of locations outside the GT boxes will also be considered.
So, we can actually view the VLR as 
an outward extension of the main distillation region.
Note that for anchor-free detectors, like FCOS, 
we can use the preset anchors on feature maps and 
do not change its regression form, 
so that the localization learning maintains to be the anchor-free type.
While for anchor-based detectors which usually set multiple anchors 
per location, 
we unfold the anchor boxes to calculate the DIoU matrix, 
and then assign their attributes.

\begin{table*}[t]
  \centering
  \caption{\textbf{Ablations}. We show ablation experiments for LD and VLR 
    on MS COCO val2017.
  } \vspace{-12pt}\label{tab:ablation}
  \subfloat[\textbf{Temperature $\tau$ in LD}: 
  The generalized Softmax function with large $\tau$ brings considerable gains. 
  We set $\tau=10$ by default. 
  The teacher is ResNet-101 and the student is ResNet-50.]{
  \tablestyle{3pt}{1.2}
  \begin{tabular}{c|c c c|c c c}
  $\tau$ & AP & AP$_{50}$ & AP$_{75}$ & AP$_{S}$ & AP$_{M}$ & AP$_{L}$\\ \shline
   -- & 40.1 & 58.2 & 43.1 & 23.3 & 44.4 & 52.5 \\ \hline
   1  & 40.3 & 58.2 & 43.4 & 22.4 & 44.0 & 52.4 \\
   5  & 40.9 & 58.2 & 44.3 & 23.2 & \hlt{45.0} & 53.2\\
   10 & \hlt{41.1} & \hlt{58.7} & \hlt{44.9} & \hlt{23.8} & 44.9 & \hlt{53.6}\\
   15 & 40.7 & 58.5 & 44.2 & 23.5 & 44.3 & 53.3 \\
   20 & 40.5 & 58.3 & 43.7 & \hlt{23.8} & 44.1 & 53.5
  \end{tabular}} \hfill
  \subfloat[\textbf{LD \vs~Pseudo BBox Regression \cite{chen2017learning}}: 
    The localization knowledge can be more efficiently transferred by 
    our LD compared to the pseudo bbox regression. 
    The teacher is ResNet-101 and the student is ResNet-50.]{
  \tablestyle{3pt}{1.2}
  \begin{tabular}{c|c c c|c c c}
  $\varepsilon$ & AP &AP$_{50}$&AP$_{75}$&AP$_{S}$&AP$_{M}$&AP$_{L}$ \\ \shline
   --  & 40.1 & 58.2 & 43.1 & 23.3 & 44.4 & 52.5 \\ \hline
   0.1 & 40.5 & 58.3 & 43.8 & 23.0 & 44.2 & 52.7\\
   0.2 & 40.2 & 58.2 & 43.6 & 23.1 & 44.0 & 53.0 \\
   0.3 & 40.1 & 58.4 & 43.1 & 23.6 & 43.9 & 52.5 \\
   0.4 & 40.3 & 58.4 & 43.4 & 22.8 & 44.0 & 52.6 \\
   \hline
   LD  &\hlt{41.1}&\hlt{58.7}&\hlt{44.9}&\hlt{23.8}&\hlt{44.9}&\hlt{53.6}\\
  \end{tabular}}\hfill
  \subfloat[\textbf{Role of $\gamma$ in VLR}: 
    Conducting LD on valuable localization region has a positive effect 
    on performance. 
    We set $\gamma=0.25$ by default.
    The teacher is ResNet-101 and the student is ResNet-50.]{
  \tablestyle{3pt}{1.2}
  \begin{tabular}{c|c c c|c c c}
  $\gamma$ & AP &AP$_{50}$&AP$_{75}$& AP$_{S}$ & AP$_{M}$ & AP$_{L}$ \\ \shline
   --  & 40.1 & 58.2 & 43.1 & 23.3 & 44.4 & 52.5 \\ \hline
   1   & 41.1 & 58.7 & 44.9 & 23.8 & 44.9 & 53.6 \\
   0.75& 41.2 & 58.8 & 44.9 & 23.6 & 45.4 & 53.5 \\
   0.5 & 41.7 & 59.4 & 45.3 & 24.2 & 45.6 & 54.2 \\
   0.25& \hlt{41.8} & \hlt{59.5} & \hlt{45.4} & 24.2 & 45.8 & \hlt{54.9} \\
   0   & 41.7 & \hlt{59.5} & \hlt{45.4} & \hlt{24.5} & \hlt{45.9} & 54.0
  \end{tabular}}
\end{table*}

\subsection{\minor{Selective Region Distillation}}\label{sec:separate}

Given the above descriptions, the total loss of logit mimicking 
for training the student $\bm S$ can be represented as:
\begin{equation} \label{eqn:total_loss}
\begin{aligned}
\mathcal{L}=&\lambda_0\LossT{cls}(\mathcal{C}_S,\mathcal{C}^{gt})
  +\lambda_1\LossT{reg}(\mathcal{B}_S,\mathcal{B}^{gt})+
  \lambda_2\LossT{DFL}(\mathcal{B}_S,\mathcal{B}^{gt})\\
  +&\lambda_3\mathds{I}_\text{Main}\LossT{LD}(\mathcal{B}_S,\mathcal{B}_T)
  +\lambda_4\mathds{I}_\text{VL}\LossT{LD}(\mathcal{B}_S, \mathcal{B}_T)\\
  +&\lambda_5\mathds{I}_\text{Main}\LossT{KD}(\mathcal{C}_S, \mathcal{C}_T)
  +\lambda_6\mathds{I}_\text{VL}\LossT{KD}(\mathcal{C}_S, \mathcal{C}_T),
\end{aligned}
\end{equation}
where the first three terms are exactly the same to the classification 
and bounding box regression branches for any regression-based detector, 
\ie $\LossT{cls}$ is the classification loss, 
$\LossT{reg}$ is the bounding box regression loss and 
$\LossT{DFL}$ is the distribution focal loss \cite{gfocal}.
$\mathds{I}_\text{Main}$ and $\mathds{I}_\text{VL}$ are the distillation 
masks for the main distillation region and the valuable localization region, 
respectively.
$\LossT{KD}$ is the KD loss~\cite{hinton2015distilling},  
$\mathcal{C}_S$ as well as $\mathcal{C}_T$ denote the classification 
head output logits of the student and the teacher, respectively, 
and $\mathcal{C}^{gt}$ is the ground-truth class label.

All the distillation losses will be weighted by the same weight 
factors according to their types.
More clearly, the weight factor of the LD loss follows that of the 
bbox regression term and the weight factor of the KD loss follows 
that of the classification term.
Also, it is worth mentioning that the DFL loss term can be disabled 
since LD loss has sufficient guidance ability.
In addition, we can enable or disable the four types of distillation losses 
so as to distill the student in different regions selectively.

\section{Experiment} \label{sec:results}

In this section, we conduct comprehensive ablation studies and analysis to demonstrate the superiority of the proposed LD and distillation scheme on the challenging large-scale MS COCO \cite{coco} benchmark, PASCAL VOC \cite{voc}, and aerial image DOTA dataset \cite{dota}.

\subsection{Experiment Setup}

\myPara{MS COCO.}
The train2017 (118K images) is utilized for training and 
val2017 (5K images) is used for validation.
We also obtain the evaluation results on MS COCO test-dev 2019 
(20K images) by submitting to the COCO server.
The experiments are conducted under the mmDetection 
\cite{mmdetection} framework.
Unless otherwise stated, we use ResNet \cite{ResNet} with FPN \cite{FPN} 
as our backbone and neck networks, 
and the FCOS-style \cite{FCOS} anchor-free head for classification 
and localization.
The training schedule for ablation experiments is set to single-scale 
1$\times$ mode (12 epochs).
For other training and testing hyper-parameters, 
we follow exactly the GFocal \cite{gfocal} protocol, 
including QFL loss for classification and GIoU loss for bbox regression, \etc.
We use the standard COCO-style measurement, \ie average precision (AP), 
for evaluation.
All the baseline models are retrained by adopting the same settings 
so as to fairly compare them with our LD.

\myPara{PASCAL VOC.}
We also provide experimental results on another popular object detection 
benchmark, \ie PASCAL VOC \cite{voc}.
We use the VOC 07+12 training protocol, \ie the union of 
VOC 2007 trainval set and VOC 2012 trainval set (16551 images) for training, 
and VOC 2007 test set (4952 images) for evaluation.
The initial learning rate is 0.01 and the total training epochs are set to 4.
The learning rate decreases by a factor of 10 after the 3rd epoch.
For comprehensively evaluating the localization performance, 
the average precision (AP) along with 5 mAP across different IoU thresholds 
are reported, 
\ie AP$_{50}$, AP$_{60}$, AP$_{70}$, AP$_{80}$ and AP$_{90}$.

\myPara{DOTA.}
As for the evaluation of rotated LD, we report the detection results 
on the classic aerial image dataset DOTA \cite{dota}.
We follow the standard mmRotate \cite{zhou2022mmrotate} 
training and testing protocol.
The train set and validation set consist of 1403 images and 468 images, 
respectively, which are randomly selected in literature.
These huge images are cropped into smaller subimages with shape 
$600\times600$, 
which is in line with the cropping protocol in official implementation.
In practice, we obtain about 15,700 training and 5,300 validation patches.
Unless otherwise stated, all the hyper-parameters follow the default settings 
of mmRotate for a fair comparison.
We report results in terms of AP and 5 mAPs under different IoU thresholds, 
which is consistent with PASCAL VOC.
Due to the memory limitation, the teachers are ResNet-34 FPN with 
2$\times$ training schedule (24 epochs), 
and the students are ResNet-18 FPN with 1$\times$ training schedule 
(12 epochs).

\subsection{Ablation Study}

\myPara{Temperature $\tau$ in LD.}
Our LD introduces a hyper-parameter, \ie the temperature $\tau$.
\tabref{tab:ablation}(a) reports the results of LD with various temperatures, 
where the teacher model is ResNet-101 with AP 44.7 and the student model 
is ResNet-50.
Here, only the main distillation region is adopted.
Compared to the first row in \tabref{tab:ablation}(a), different temperatures
consistently lead to better results.
In this paper, we simply set the temperature in LD as $\tau=10$, 
which is fixed in all the other experiments.

\myPara{LD {\em vs}. Pseudo BBox Regression.}
The teacher bounded regression (TBR) loss~\cite{chen2017learning} 
is a preliminary attempt to enhance the student on the localization head, 
\ie the pseudo bbox regression in \figref{fig:previous},
which is represented as:
\begin{equation} \label{eq:tbr}
\begin{aligned}
  \LossT{TBR}=\lambda\LossT{reg}(\mathcal{B}^{s},\mathcal{B}^{gt}),
  \text{if}~\ell_2(\mathcal{B}^{s},\mathcal{B}^{gt})+\varepsilon 
  > \ell_2(\mathcal{B}^{t},\mathcal{B}^{gt}),\\
\end{aligned}
\end{equation}
where $\mathcal{B}^{s}$ and $\mathcal{B}^{t}$ denote the predicted boxes 
of student and teacher respectively, 
$\mathcal{B}^{gt}$ denotes the ground truth boxes, 
$\varepsilon$ is a predefined margin, 
and $\LossT{reg}$ represents the GIoU loss \cite{giou}.
Here, only the main distillation region is adopted.
From \tabref{tab:ablation}(b), we can see that the TBR loss does 
yield performance gains (+0.4 AP and +0.7 AP$_{75}$) 
when using proper threshold $\varepsilon=0.1$ in \eqref{eq:tbr}.
However, it uses the coarse bbox representation, 
which does not contain any localization uncertainty information 
of the detector, leading to sub-optimal results.
On the contrary, our LD directly produces 41.1 AP and 44.9 AP$_{75}$, 
since it utilizes the probability distribution of bounding boxes 
which contains rich localization knowledge.

\myPara{Various $\gamma$ in VLR.}
The newly introduced VLR has the parameter $\gamma$ which controls the range of VLR.
As shown in \tabref{tab:ablation}(c), 
AP is stable when $\gamma$ ranges from 0 to 0.5.
The variation in AP in this range is around 0.1.
As $\gamma$ increases, the VLR gradually shrinks to empty.
The performance also gradually drops to 41.1, 
\ie conducting LD on the main distillation region only.
The sensitivity analysis experiments on the parameter $\gamma$ indicate
that conducting LD on the VLR has a positive effect on performance.
In the rest experiments, we set $\gamma$ to 0.25 for simplicity.

\begin{table}[!t]
  \small
  \centering
  \renewcommand\arraystretch{1.1}
  \setlength{\tabcolsep}{2pt}
  \caption{Evaluation of \textbf{selective region distillation} 
    for KD and our LD. 
    The teacher-student pair is ResNet-101$\rightarrow$ResNet-50 for COCO, 
    and ResNet-101$\rightarrow$ResNet-18 for VOC 07+12.
  }
  \begin{tabular}{ccccc|c c c | c c c} \hline \hline
  \multicolumn{2}{c}{LD} && \multicolumn{2}{c|}{KD} & \multicolumn{3}{c|}{MS COCO val2017} & \multicolumn{3}{c}{VOC 07+12} \\
  \cline{1-2}\cline{4-5} \cline{6-8} \cline{9-11} 
    Main & VLR && Main & VLR & AP & AP$_{50}$ & AP$_{75}$ & AP & AP$_{50}$ & AP$_{75}$ \\ \hline
          &       &&       &       & 40.1     & 58.2     & 43.1 & 51.8 & 75.8 & 56.3 \\ \hline
    \chmk &       &&       &       & 41.1     & 58.7     & 44.9 & 53.0 & 75.9 & 57.6 \\
    \chmk & \chmk &&       &       & 41.8     & 59.5     & 45.4 & 53.4 & 76.3 & \hlt{58.3} \\
    \chmk & \chmk && \chmk &       &\hlt{42.1}&\hlt{60.3}&\hlt{45.6} & 53.1 & 76.8 & 57.6\\
    \chmk & \chmk && \chmk & \chmk & 42.0     & 60.0     & 45.4  & \hlt{53.7} & \hlt{77.3} & 58.2 \\
    \hline
    
    \hline
  \end{tabular}
  \label{tab:separateregion} 
\end{table}

\myPara{Selective Region Distillation.}
\minor{
There are several interesting observations regarding the roles of KD and LD 
and their preferred regions.
We report the relevant ablation study results in \tabref{tab:separateregion}, 
where ``Main"  means that the logit mimicking is conducted 
on the main distillation region, 
\ie the positive locations of label assignment, 
and "VLR" denotes the valuable localization region.
For MS COCO, it can be seen that conducting ``Main LD", 
``VLR LD", and ``Main KD" all benefits the student's performance.
This indicates that the main distillation regions contain the valuable knowledge
for both classification and localization and the classification KD benefits less
compared to LD.
Then, we impose the classification KD on a larger range, \ie the VLR.
However, we observe that further incorporating ``VLR KD" yields no improvement
(the last two rows of \tabref{tab:separateregion}).
This is the main reason why we adopt the proposed selective region distillation for COCO.
}

Next, we check the roles of KD and LD on PASCAL VOC.
\tabref{tab:separateregion} shows that it is beneficial to transfer the localization knowledge to both the main distillation region and the VLR.
However, due to the different knowledge distribution patterns, it shows a similar degradation of the classification KD.
Comparing the 3rd row and the 4th row of \tabref{tab:separateregion}, ``Main KD'' leads to a performance drop,
while ``VLR KD" produces a positive effect to the student.
This indicates that the selective region distillation can take the advantages of both KD and LD 
on their respective favorable regions.

  \begin{table}[!t]\small
    \centering
    \renewcommand\arraystretch{1.1}
    \setlength{\tabcolsep}{5pt}
    \caption{Quantitative results of LD for lightweight detectors. The teacher is ResNet-101. The results are reported on MS COCO val2017.}
    
        \begin{tabular}{c|c|c c c|c c c}
          \hline
  
          \hline
          Student & LD & AP & AP$_{50}$ & AP$_{75}$ & AP$_{S}$ & AP$_{M}$ & AP$_{L}$  \\
          \hline
          \multirow{2}{*}{ResNet-18}  &  & 35.8  & 53.1 & 38.2 &  18.9 &  38.9 &  47.9\\
          & \chmk & 37.5  & 54.7 & 40.4 & 20.2 & 41.2 & 49.4 \\
                  \hline
          \multirow{2}{*}{ResNet-34}&  & 38.9 & 56.6 & 42.2 &  21.5 &  42.8 & 51.4 \\
          &  \chmk &  41.0 & 58.6 & 44.6 & 23.2 & 45.0 & 54.2 \\
          \hline
          \multirow{2}{*}{ResNet-50}&   & 40.1  & 58.2 & 43.1 &  23.3 & 44.4  & 52.5 \\
          & \chmk  & 42.1  & 60.3 & 45.6 & 24.5 & 46.2  & 54.8 \\
          \hline
  
          \hline
      \end{tabular}
      \label{tab:testdev1}
  \end{table}
  
  \myPara{LD for Lightweight Detectors.}
  \tabref{tab:testdev1} reports the results of our distillation scheme \minor{(``Main LD + VLR LD + Main KD" on COCO)}, where a series of lightweight students are distilled, including ResNet-18, ResNet-34, and ResNet-50.
  For all given students, our LD can stably improve the detection performance without any bells and whistles.
  From these results, we can see that our LD improves the students ResNet-18, ResNet-34, ResNet-50 by +1.7, +2.1, +2.0 in AP, and +2.2, +2.4, +2.5 in AP$_{75}$, respectively.
  %
  %
  %

  \begin{table}[!tb]\small
    \centering
    \renewcommand\arraystretch{1.1}
    \setlength{\tabcolsep}{4pt}
    \caption{Quantitative results of LD on various popular dense object detectors. The teacher is ResNet-101 and the student is ResNet-50. The results are reported on MS COCO val2017.}		
        \begin{tabular}{c|c|c c c|c c c}
          \hline
  
          \hline
          Student & LD & AP & AP$_{50}$ & AP$_{75}$ & AP$_{S}$ & AP$_{M}$ & AP$_{L}$  \\
          \hline
          \multirow{2}{*}{RetinaNet \cite{lin2017focal}}  &  & 36.9  & 54.3 & 39.8 &  21.2 &  40.8 &  48.4\\
          & \chmk & 39.0  & 56.4 & 42.4 &  23.1 &  43.2 &  51.1\\
                  \hline
          \multirow{2}{*}{FCOS \cite{FCOS}}&  & 38.6 & 57.2 & 41.5 &  22.4 &  42.2 & 49.8 \\
          &  \chmk &  40.6 & 58.4 & 44.1 &  24.3 &  44.1 & 52.3 \\
          \hline
          \multirow{2}{*}{ATSS \cite{ATSS}}&   & 39.2  & 57.3 & 42.4 &  22.7 & 43.1  & 51.5 \\
          & \chmk  & 41.6  & 59.3 & 45.3 & 25.2 & 45.2  & 53.3 \\
          \hline
  
          \hline
        \end{tabular}
      \label{tab:other}
  \end{table}
  
  \myPara{Application to Other Dense Object Detectors.}
  Our LD can be flexibly applied to other dense object detectors, including
  either anchor-based or anchor-free types.
  We employ LD with the divide-and-conquer distillation scheme to several recently popular detectors, such as RetinaNet \cite{lin2017focal} (anchor-based), FCOS \cite{FCOS} (anchor-free) and ATSS \cite{ATSS} (anchor-based).
  According to the results in \tabref{tab:other}, we can see that 
  our LD can consistently improve the baselines by around 2 AP scores.
  
  \myPara{Arbitrary-Oriented Object Detectors.}
  As a direct extension of our LD, the rotated bounding box requires an additional probability distribution, i.e., the rotated angle distribution.
  We make the necessary and minimum modification to two arbitrary-oriented object detectors, 1) the foundation of dense regression-based rotated detector---Rotated-RetinaNet \cite{lin2017focal} and 2) the recently popular 2D Gaussian distribution modeling detector---GWD \cite{GWD}.
  We follow the mmRotate \cite{zhou2022mmrotate} training and testing protocols.
  We use ResNet-34 as the teacher and ResNet-18 as the student for GPU memory saving.
  The results are reported on the validation set of DOTA-v1.0 \cite{dota}.
  
  \begin{table}[!tb]\small
    \centering
    \renewcommand\arraystretch{1.1}
    \setlength{\tabcolsep}{4.5pt}
    \caption{Quantitative results of rotated LD on the popular arbitrary-oriented object detectors. The teacher is ResNet-34 and the student is ResNet-18. The results are reported on the validation set of DOTA-v1.0.}		
        \begin{tabular}{c|c|c c c c c c}
          \hline
  
          \hline
          Student & AP & AP$_{50}$ & AP$_{60}$ & AP$_{70}$ & AP$_{80}$ & AP$_{90}$  \\
          \hline
          R-RetinaNet \cite{lin2017focal}   & 33.7  & 58.0 & 54.5 &  42.3 &  22.9 &  4.7\\
          LD (ours) & 39.1  & 63.8 & 61.1 &  48.8 &  28.7 &  8.8\\
                  \hline
          GWD \cite{GWD}  & 37.1 & 63.1 & 60.1 &  46.7 &  24.7 & 6.2 \\
          LD (ours) &  40.2 & 66.4 & 63.6 &  50.3 &  28.2 & 8.5 \\
          \hline
  
          \hline
        \end{tabular}
      \label{tab:rotated}
  \end{table}
  
  The results have been shown in \tabref{tab:rotated}, which demonstrates that 
  our LD can also be successfully applied to rotated object detectors and 
  attain considerable improvement in aerial image detection.
  Particularly, we obtain impressive improvements for the mAP under more rigorous IoU thresholds, e.g., AP$_{70}$, AP$_{80}$, AP$_{90}$.
  This shows the excellent compatibility of our LD, which can be applied to not only horizontal bounding boxes but also the rotated ones.
  In addition, it is worth mentioning that our LD does not rely on the representations of bounding boxes and the optimization way of modeling (IoU-based loss for horizontal bounding box prediction \cite{giou,diou} and 2D Gaussian modeling for rotated bounding box prediction \cite{GWD}).
  
  \subsection{Logit Mimicking {\em v.s.} Feature Imitation.}
  Thus far, we have validated the effectiveness of our LD and the selective region distillation in distilling different types of object detectors.
  The proposed LD along with the classification KD provides a unified
  logit mimicking framework.
  It naturally raises several interesting questions:
  \begin{itemize}
      \item In terms of detection performance, how does logit mimicking perform compared to feature imitation? Does feature imitation stay ahead of logit mimicking?
      \item What are the characteristics of these two different distillation techniques? Are the deep feature representations and logits learned different?
  \end{itemize}
  In this subsection, we shall provide answers to the above questions.
  
  \myPara{Quantitative Comparison on Numerical Results.}
  We first compare our proposed LD with several state-of-the-art feature imitation methods.
  We adopt the selective region distillation, i.e., 
  performing KD and LD on the main distillation region, and performing LD on the VLR.
  Since modern detectors are usually equipped with FPN~\cite{FPN}, following previous works \cite{wang2019distilling,GIbox,defeat}, we re-implement their methods and impose all the feature imitations on multi-level FPN for a fair comparison.
  Here, ``FitNets" \cite{FitNets} distills 
  the whole feature maps.
  ``DeFeat" \cite{defeat} means the loss weights of feature imitation outside the GT boxes are larger than those inside the GT boxes.
  ``Fine-Grained" \cite{wang2019distilling} distills the deep features on the close anchor box locations.
  ``GI Imitation" \cite{GIbox} selects the distillation regions according to the discriminative predictions of the student and the teacher.
  ``Inside GT Box" means we select the ground-truth boxes 
  with the same stride on the FPN layers as the feature imitation regions.
  ``Main Region" means we imitate the features within the main distillation region.

  \begin{table}[!t]\small
    
    \centering
    
    \renewcommand\arraystretch{1.15}
    \setlength{\tabcolsep}{3pt}
    \caption{\textbf{Logit Mimicking {\em vs}. Feature Imitation.} ``Ours" means we use the selective region distillation, i.e., ``Main LD + VLR LD + Main KD". \minor{``*" denotes we remove the ``Main KD".} The teacher is ResNet-101 and the student is ResNet-50 \cite{ResNet}. The results are reported on MS COCO val2017.}
        \begin{tabular}{l|c c c|c c c}
          \hline
          
          \hline
          Method & AP & AP$_{50}$ & AP$_{75}$ & AP$_{S}$ & AP$_{M}$ & AP$_{L}$  \\
                  \hline
                  
          Baseline (GFocal \cite{gfocal}) & 40.1 & 58.2 & 43.1 & 23.3 & 44.4 & 52.5 \\
          \hline
                  FitNets \cite{FitNets} & 40.7 & 58.6 & 44.0 & 23.7 & 44.4 & 53.2 \\
                  Inside GT Box  & 40.7 & 58.6 & 44.2 & 23.1 & 44.5 & 53.5 \\
                  Main Region  & 41.1 & 58.7 & 44.4 & 24.1 & 44.6 & 53.6 \\
                  Fine-Grained \cite{wang2019distilling} & 41.1 & 58.8 & 44.8 & 23.3 & 45.4 & 53.1 \\
                  DeFeat \cite{defeat} & 40.8 & 58.6 & 44.2 & 24.3 & 44.6 & 53.7 \\
                  GI Imitation \cite{GIbox} & 41.5 & 59.6 & 45.2 & 24.3 & 45.7 & 53.6 \\
                  \hline
                  Ours  & 42.1 & 60.3 & 45.6 & 24.5 & 46.2 & 54.8 \\
                  Ours + FitNets  & 42.1 & 59.9 & 45.7 & 25.0 & 46.3 & 54.4 \\
                  Ours + Inside GT Box  & 42.2 & 60.0 & 45.9 & 24.3 & 46.3 & 55.0 \\
                  Ours + Main Region  & 42.1  & 60.0 & 45.7 & 24.6 & 46.3 & 54.7 \\
                  Ours + Fine-Grained  & 42.4 & 60.3 & 45.9 & 24.7 & 46.5 & 55.4 \\
                  \minor{Ours* + Fine-Grained}  & 42.1 & 59.7 & 45.6 & 24.8 & 46.1 & 54.8 \\
                  Ours + DeFeat  & 42.2 & 60.0 & 45.8 & 24.7 & 46.1 & 54.4 \\
                  Ours + GI Imitation  & 42.4 & 60.3 & 46.2 & 25.0 & 46.6 & 54.5 \\
                  \hline
          
          \hline
        \end{tabular}
    \label{tab:compare}
  \end{table}
  \begin{figure}[t]
    \centering
    \includegraphics[width=0.48\textwidth]{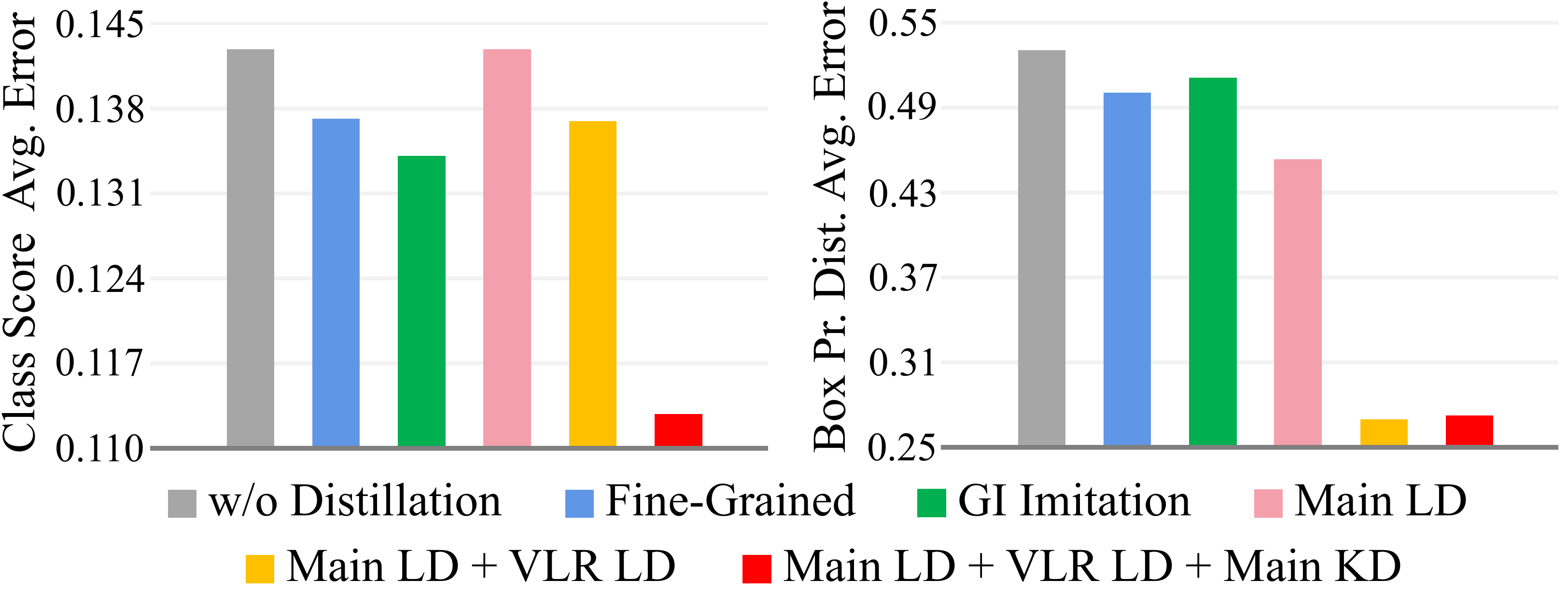}
    \caption{Visual comparisons of SOTA feature imitation and our LD. We show the average L1 error of classification scores and box probability distributions between teacher and student at the P4, P5, P6 and P7 FPN levels. The teacher is ResNet-101 and the student is ResNet-50. The results are evaluated on MS COCO val2017.}
    \label{fig:error}
  \end{figure} 
  \begin{figure*}[!t]
    \centering
    \includegraphics[width=\textwidth]{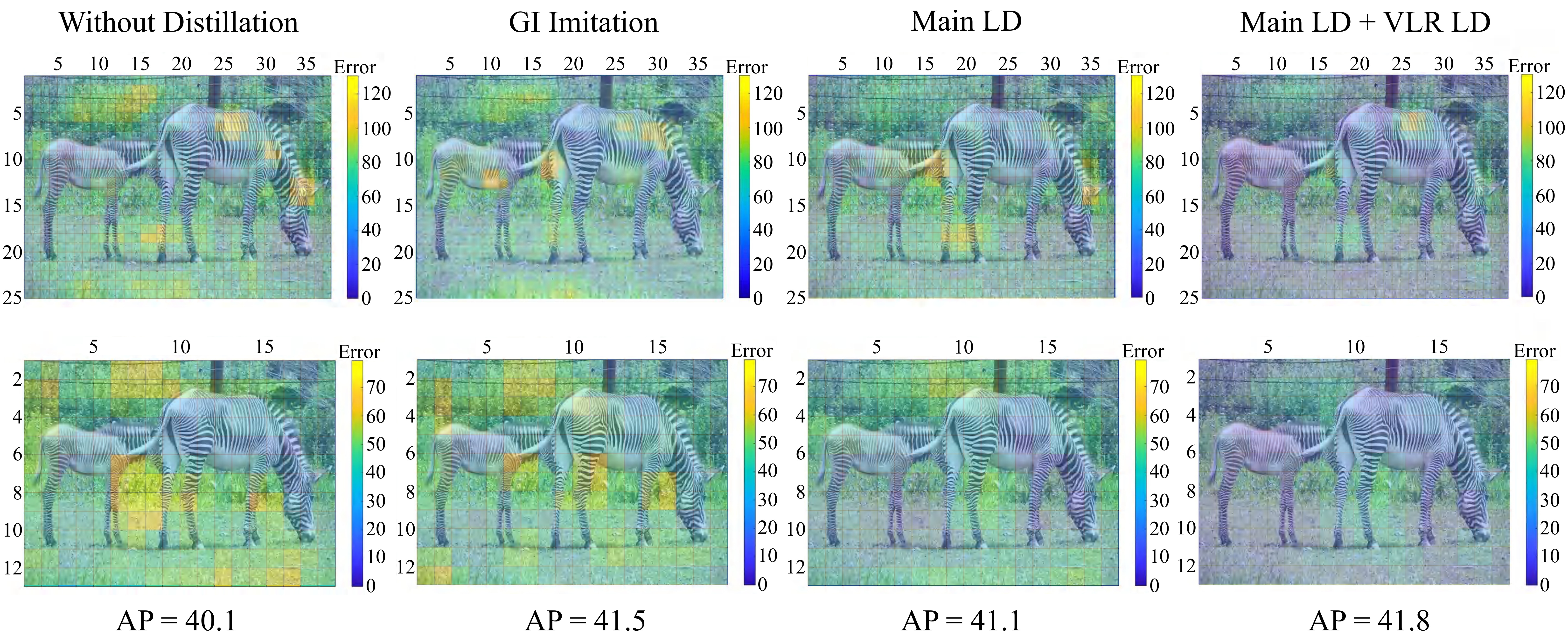}
    \vspace{-20pt}
    \caption{Visual comparisons between the state-of-the-art feature imitation and our LD. We show the per-location L1 error summation of the localization head logits between the teacher and the student as the P5 (first row) and P6 (second row) FPN levels. The teacher is ResNet-101 and the student is ResNet-50. We can see that compared to the GI imitation \cite{GIbox}, our method (``Main LD + VLR LD") can
    significantly reduce the errors for almost all the locations.
    Darker is better. Best viewed in color.}
    \label{fig:visual}
  \end{figure*}
  From \tabref{tab:compare}, we can see that distillation
  within the whole feature maps attains +0.6 AP gains.
  By setting a larger loss weight for the locations outside the GT boxes (DeFeat \cite{defeat}), the performance is slightly better than that
  using the same loss weight for all locations.
  Fine-Grained \cite{wang2019distilling} focusing on the locations near GT boxes, produces 41.1 AP, which is comparable to the results of feature imitation using the Main Region.
  GI imitation \cite{GIbox} searches the discriminative patches for feature imitation and gains 41.5 AP.
  Due to the large gap in predictions between student and teacher,
  the imitation regions may appear anywhere.
  
  Despite the notable improvements of these feature imitation methods, they do not explicitly consider the knowledge distribution patterns.
  On the contrary, our method can transfer the knowledge via a selective region distillation, which directly produces 42.1 AP.
  It is worth noting that our method operates on logits instead of deep features, indicating that our LD is a critical component for logit mimicking to outperform the feature imitation.
  Moreover, our method is orthogonal to the aforementioned feature imitation methods.
  \tabref{tab:compare} shows that with these feature imitation methods, our performance can be further improved.
  Particularly, with GI imitation, we improve the strong GFocal baseline by +2.3 AP and +3.1 AP$_{75}$.

  \myPara{Teacher-Student Error Comparison.}
  We first check the average teacher-student errors of the classification scores and the box probability distributions, as shown in Fig \ref{fig:error}.
  One can see that the Fine-Grained feature imitation \cite{wang2019distilling} and GI imitation \cite{GIbox} reduce the two errors as expected, since the classification knowledge and localization knowledge are mixed on feature maps.
  Our ``Main LD" and ``Main LD + VLR LD" have comparable or larger classification score average errors than Fine-Grained \cite{wang2019distilling} and GI imitation \cite{GIbox} but lower box probability distribution average errors.
  This indicates that these two settings
  with only LD can significantly reduce the box probability distribution distance between the teacher and the student
  but they cannot reduce this error for the classification head.
  If we impose the classification 
  KD on the main distillation region, 
  yielding ``Main LD + VLR LD + Main KD", 
  both the classification score average error and the box probability distribution average error can be reduced.
  
  We also visualize the L1 error summation of the localization head logits between the student and the teacher for each location at the P5 and P6 FPN levels.
  As shown in \figref{fig:visual}, comparing to ``Without Distillation", we can see that the GI imitation \cite{GIbox} does decrease the localization discrepancy between the teacher and the student.
  Notice that we particularly choose a model (``Main LD + VLR LD") with slightly better AP performance than GI imitation for visualization.
  Our method can clearly reduce this error and alleviate the localization ambiguity.
  
\begin{figure*}[!t]
  \centering
  \setlength{\abovecaptionskip}{5pt}
  \includegraphics[width=1\textwidth]{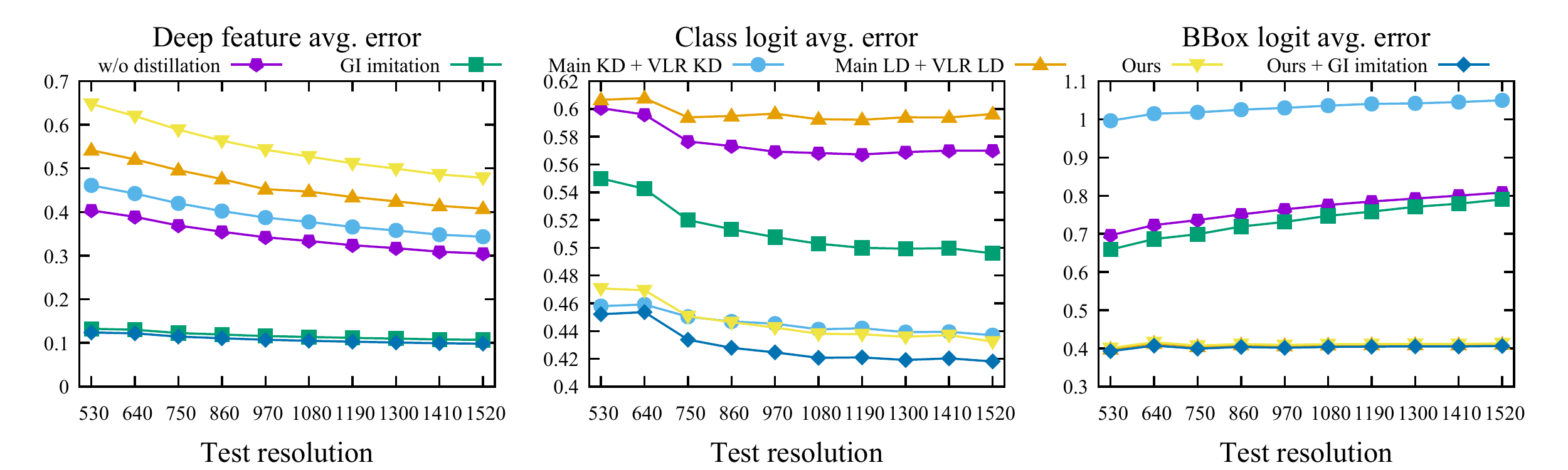}
  \vspace{-17pt}
  \caption{Average teacher-student error on (left) deep feature representation, (middle) class logits, and (right) bbox logits. ``Ours'' denotes ``Main LD + VLR LD + Main KD". The curves are evaluated on MS COCO val2017.}
  \label{fig:curve}
  
\end{figure*}
  
\begin{table}[!t]\small
  \centering
  \renewcommand\arraystretch{1.2}
  \setlength{\tabcolsep}{4pt}
  \caption{The average Pearson correlation coefficient between the teacher-student pair. 'GI': GI imitation. 'Ours': our logit mimicking scheme with the selective region distillation. The results are evaluated on MS COCO val2017.}
    \begin{tabular}{c|cccc}
      \hline
        
      \hline
      & w/o distillation & GI& Ours & Ours + GI\\
      \hline
        deep features & -0.0042 & 0.8175 & -0.0031 & 0.8373\\
        bbox logits & 0.9222 & 0.9326 & 0.9733 & 0.9745\\
      \hline
      
      \hline
    \end{tabular}
    \label{tab:pearson}
\end{table}

In \figref{fig:curve}, we plot the average errors between the student and the teacher in terms of deep feature, class logit and bbox logit, respectively.
It can be seen that these three types of errors show an almost consistent trend as the test resolution changes. 
Interestingly, we find that even though the logit mimicking can shrink the errors of both the bbox logits and the classification ones, it learns complete different feature representations from the teacher's.
From the left side of \figref{fig:curve}, our method enlarges the distance between the student's feature representations and those of the teacher.
Moreover, \tabref{tab:pearson} shows that the logit mimicking produces a nearly zero Pearson correlation coefficient for the feature representations between the teacher-student pair.
This indicates that if the student is only trained with logit mimicking, it produces a far different and nonlinearly correlated feature representation to teacher's.
%
Be that as it may, we can still attain well-performed logits for good generalization.
The last column of \tabref{tab:pearson} and \figref{fig:curve} show that the logit mimicking is able to approach the teacher's logits not only in distance but also in linear correlation.
%

\myPara{AP Landscape.}
Distilling an object detector from either the feature level or the logit level is a high-dimensional non-convex optimization problem, which is easy in practice but hard in theory.
To better understand the behavior of logit mimicking and feature imitation, we present a new visualization method, termed AP landscape, which is especially designed for object detection to observe the AP changes caused by minute perturbations in the learnt feature representations.
A canonical approach was taken in \cite{losslandscape}, who studied the loss surface visualization by linearly interpolating the parameters of two networks.
\begin{figure}[!t]
  \flushright
  \small
  \begin{overpic}[width=1\columnwidth]{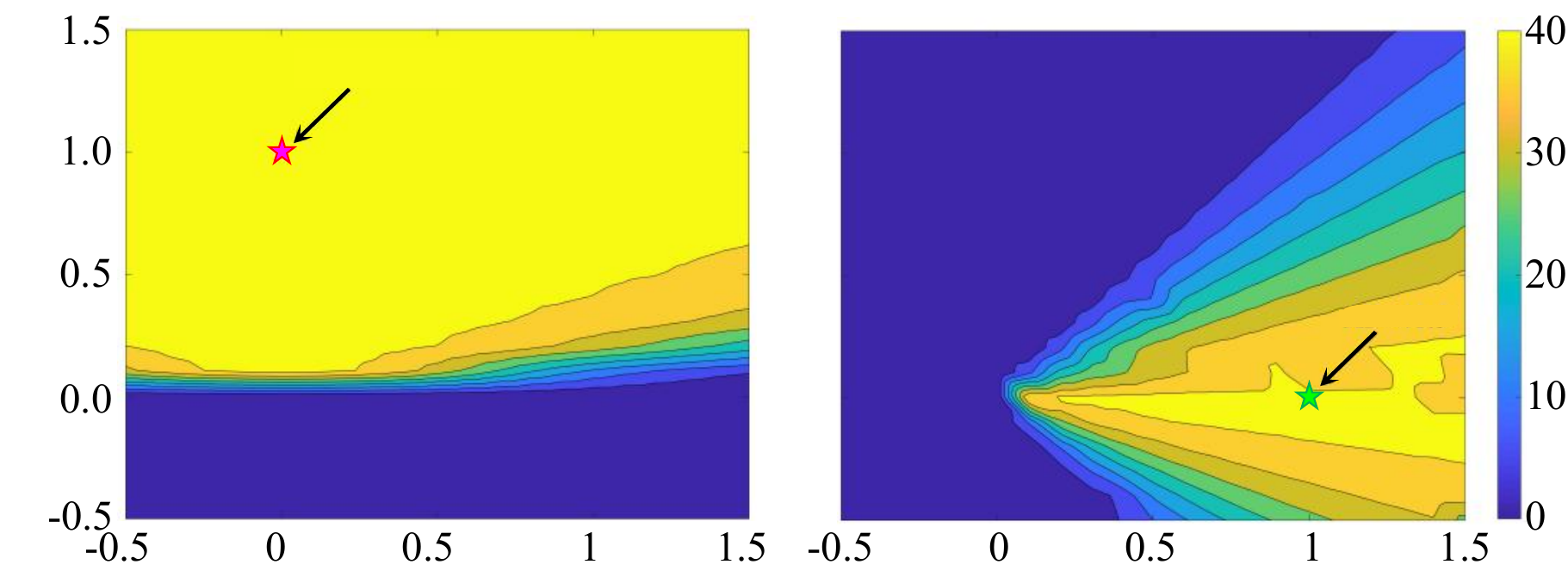}
      \put(26.3, 36.5){$\alpha$}
      \put(73.2, 36.5){$\alpha$}
      \put(0, 18){$\beta$}
      \put(23.9, -3.2){Ours}
      \put(64, -3.2){GI Imitation}
      \put(19,31.6){\footnotesize{AP=42.1}}
      \put(81.4,16){\footnotesize{AP=41.5}}
  \end{overpic} 
  \vspace{-10pt}
  \caption{
    The 2D contour plots of AP landscapes in feature subspace. 
    The AP landscapes are evaluated on MS COCO val2017.
  }\label{fig:APlandscape}
\end{figure}
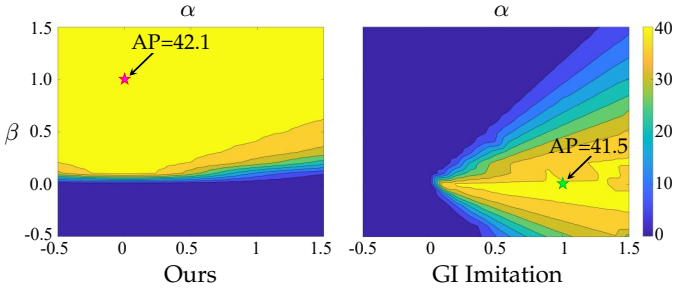

In our visualization, we are particularly curious about the empirical characterization of the feature representations and how they affect the final performance.
Considering two feature representations $M_f$, $M_l$ which are learnt by the detectors trained with feature imitation and logit mimicking, respectively, we visualize the AP landscapes within the 2D projected space $M_f\oplus M_l$.
We use two scalar parameters $\alpha$ and $\beta$ to obtain a new feature representation by using the weighted sum $M(\alpha,\beta)=\alpha M_f+\beta M_l$.
Note that when $\alpha=0$ and $\beta=1$, it represents that the feature representations are predicted by the logit mimicking method and inversely the feature imitation when $\alpha=1$ and $\beta=0$.
Then, we feed $M(\alpha,\beta)$ to the downstream heads and plot the final AP score.
Due to the computational burden, we set $\alpha,\beta \in [-0.5,1.5] $ to visualize the 2D AP landscapes.

From \figref{fig:APlandscape}, we see that logit mimicking learns
robust feature representations, i.e., the red pentagram at $(0,1)$, which is surrounded by a flat and well-performed region of AP score.
Second, we observe that the GI imitation produces a much sharper AP landscape than logit mimicking.
We attribute the landscape sharpness of the GI imitation to 
the hard $l_2$ loss supervision.
In this case, it is hard for the student to imitate the high-level and advanced feature representations from the teacher,
which corresponds to a heavy detector with a longer training schedule and higher accuracy.
On the contrary, the logit mimicking gives the feature representations much more liberty to learn, leading to a better generalization.
As shown in \figref{fig:training_difficulty}, logit mimicking can also reduce the optimization difficulty in the early training stage, while feature imitation converges slower and has a poor generalization in the early training stage.

\begin{figure}[!t]
  \centering
  \vspace{-10pt}
  \includegraphics[width=0.9\linewidth]{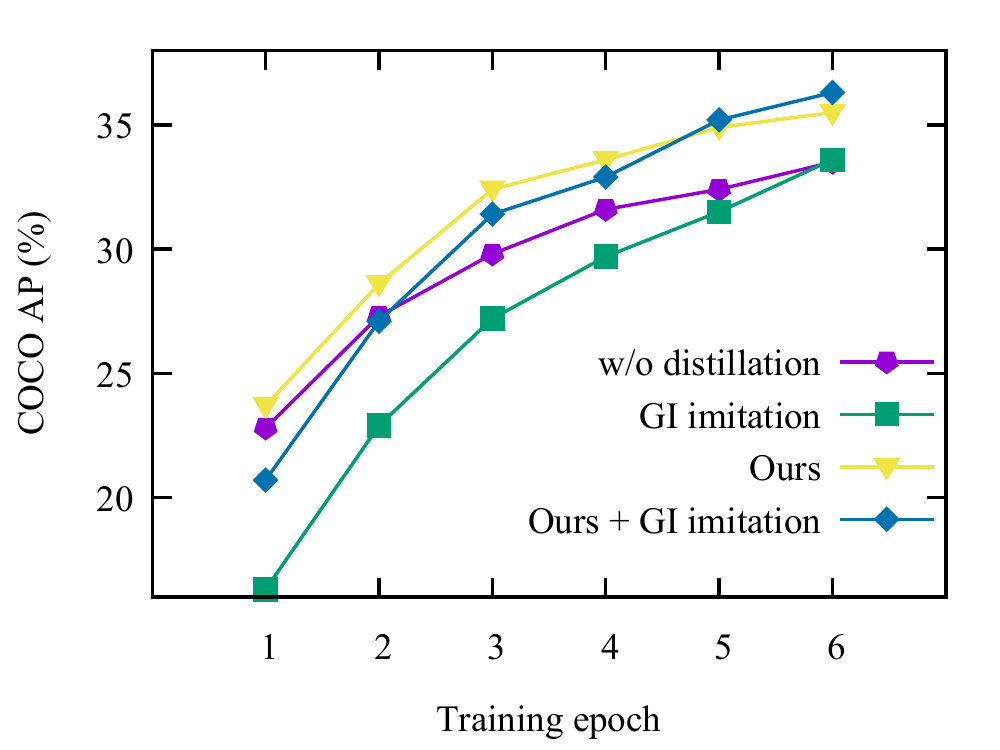}\\
  \vspace{-8pt}
  \caption{The average precision (AP) during the early training stage. 
    The feature imitation significantly slows down the convergence and gets 
    a sub-optimal generalization. 
    Logit mimicking (Ours) can reduce the training difficulty in the 
    early training stage.
  }\label{fig:training_difficulty}
\end{figure}

\myPara{Summary.}
Based on the above results and observations, we can draw the following conclusions:
\begin{itemize}
    \item Logit mimicking can outperform feature imitation in object detection when the localization knowledge distillation is explicitly distilled.
    \item Feature imitation can increase the consistency of the feature representations between the teacher-student pair, but come some drawbacks such as less feature robustness and slow training convergence. 
    Logit mimicking with the selective region distillation can significantly increase the consistency of the logits between the teacher-student pair, keep the learning liberty of features, and thereby speed up training process and benefit the KD performance more.
    This indicates that the consistency of feature representations between the teacher-student pair is not the crucial factor of improving the KD performance.
\end{itemize}

\subsection{Comparison with the State-of-the-Arts}

We compare our LD with the state-of-the-art dense object detectors by using our LD to further boost GFocalV2 \cite{gfocalv2}.
For COCO val2017, since most previous works use ResNet-50-FPN backbone with the single-scale $1\times$ training schedule (12 epochs) for validation, we also report the results under this setting for a fair comparison.
For COCO test-dev 2019, following a previous work \cite{gfocalv2}, the LD models with the $1333\times[480:960]$ multi-scale $2\times$ training schedule (24 epochs) are included.
The training is carried on a machine node with 8 GPUs with a batch size of 2 per GPU and initial learning rate 0.01 for a fair comparison.
During inference, single-scale testing ([$1333\times 800$] resolution) is adopted.
For different students ResNet-50, ResNet-101 and 
ResNeXt-101-32x4d-DCN \cite{xie2017aggregated,DCNv2}, 
we also choose different networks ResNet-101, 
ResNet-101-DCN 
and Res2Net-101-DCN \cite{gao2019res2net} as their teachers, 
respectively.

\begin{table}[t!]\small
  \caption{Comparison with state-of-the-art methods on COCO 
    \textit{val2017} and \textit{test-dev2019} . 
    \textbf{TS}: Traning Schedule. 
    '1$\times$': single-scale training 12 epochs. 
    '2$\times$': multi-scale training 24 epochs.
  }\label{tab:SOTA}
  \renewcommand\arraystretch{1.15}
  \setlength{\tabcolsep}{3.8pt}
  \centering
  \begin{tabular}{l|c|ccc|ccc} \hline \hline
    \textbf{Method} & \textbf{TS} & AP & AP$_{50}$ & AP$_{75}$ & AP$_{\it S}$ & AP$_{\it M}$ & AP$_{\it L}$ \\  \hline \hline
      \multicolumn{8}{c}{\textbf{ResNet-50 backbone on val2017}} \\\hline
      RetinaNet \cite{lin2017focal}  & 1$\times$ & 36.9  & 54.3 & 39.8 &  21.2 &  40.8 &  48.4\\
      FCOS\cite{FCOS}		& 1$\times$	& 38.6 & 57.2 & 41.5 &  22.4 &  42.2 & 49.8 \\
      SAPD \cite{zhu2019soft} & 1$\times$ & 38.8 & 58.7  & 41.3  & 22.5  & 42.6 & 50.8 \\
      ATSS \cite{ATSS} & 1$\times$  & 39.2  & 57.3 & 42.4 &  22.7 & 43.1  & 51.5 \\ 
      BorderDet \cite{qiu2020borderdet} & 1$\times$  & 41.4 & 59.4 & 44.5 & 23.6 & 45.1 & 54.6 \\
      AutoAssign \cite{zhu2020autoassign} & 1$\times$ & 40.5 & 59.8 & 43.9 & 23.1 & 44.7 & 52.9 \\
      PAA \cite{kim2020probabilistic} & 1$\times$ & 40.4 & 58.4 & 43.9 & 22.9 & 44.3 & 54.0 \\
      OTA \cite{OTA} & 1$\times$ & 40.7 & 58.4 & 44.3 & 23.2 & 45.0 & 53.6 \\
      GFocal \cite{gfocal} & 1$\times$ & 40.1 & 58.2 & 43.1 & 23.3 & 44.4 & 52.5 \\
      GFocalV2 \cite{gfocalv2} & 1$\times$ & 41.1 & 58.8 & 44.9 & 23.5 & 44.9 & 53.3 \\  
      LD (ours) & 1$\times$ & \textbf{42.7} & \textbf{60.2} & \textbf{46.7} & \textbf{25.0} & \textbf{46.4} & \textbf{55.1} \\ 
      \hline
      \multicolumn{8}{c}{\textbf{ResNet-101 backbone on test-dev 2019}} \\\hline
      RetinaNet \cite{lin2017focal}  & 2$\times$ & 39.1 & 59.1 & 42.3 & 21.8 & 42.7 & 50.2\\
      FCOS\cite{FCOS}		& 2$\times$	& 41.5 & 60.7 & 45.0 & 24.4 & 44.8 & 51.6 \\
      SAPD \cite{zhu2019soft} & 2$\times$ & 43.5 & 63.6 & 46.5 & 24.9 & 46.8 & 54.6 \\
      ATSS \cite{ATSS} & 2$\times$  & 43.6 & 62.1 & 47.4 & 26.1 & 47.0 & 53.6 \\ 
      BorderDet \cite{qiu2020borderdet} & 2$\times$  & 45.4 & 64.1 & 48.8 & 26.7 & 48.3 & 56.5 \\
      AutoAssign \cite{zhu2020autoassign} & 2$\times$ & 44.5 & 64.3 & 48.4 & 25.9 & 47.4 & 55.0 \\
      PAA \cite{kim2020probabilistic} & 2$\times$ & 44.8 & 63.3 & 48.7 & 26.5 & 48.8 & 56.3 \\
      OTA \cite{OTA} & 2$\times$ & 45.3 & 63.5 & 49.3 & 26.9 & 48.8 & 56.1\\
      GFocal \cite{gfocal} & 2$\times$ & 45.0 & 63.7 & 48.9 & 27.2 & 48.8 & 54.5\\
      GFocalV2 \cite{gfocalv2} & 2$\times$ & 46.0 & 64.1 & 50.2 & 27.6 & 49.6 & 56.5 \\  
      LD (ours) & 2$\times$ & \hlt{47.1} & \hlt{65.0} & \hlt{51.4} & \hlt{28.3} & \hlt{50.9} & \hlt{58.5} \\  
      \hline
      \multicolumn{8}{c}{\textbf{ResNeXt-101-32x4d-DCN backbone on test-dev 2019}} \\\hline
      SAPD \cite{zhu2019soft} & 2$\times$ & 46.6 & 66.6 & 50.0 & 27.3 & 49.7 & 60.7\\
      GFocal \cite{gfocal} & 2$\times$ & 48.2 & 67.4 & 52.6 & 29.2 & 51.7 & 60.2\\
      GFocalV2 \cite{gfocalv2} & 2$\times$ & 49.0 & 67.6 & 53.4 & 29.8 & 52.3 & 61.8  \\  
      LD (ours) & 2$\times$ & \hlt{50.5} & \hlt{69.0} & \hlt{55.3} & \hlt{30.9} & \hlt{54.4} & \hlt{63.4} \\  
      \hline
      
      \hline
  \end{tabular}
  \end{table}
  
  \tabref{tab:SOTA} reports the quantitative results.
  It can be seen that our LD improves the AP score of the SOTA GFocalV2
  by +1.6 and the AP$_{75}$ score by +1.8 when using the ResNet-50-FPN backbone.
  When using the ResNet-101-FPN and ResNeXt-101-32x4d-DCN with multi-scale $2\times$ training, we achieve the highest AP scores, 47.1 and 50.5 , which outperform all existing dense object detectors under the same backbone, neck and test settings.
  More importantly, our LD does not introduce any additional network parameters or computational overhead and hence can guarantee exactly the same inference speed as GFocalV2.

\section{Conclusion}\label{sec:conclusion}
\minor{In this paper, we propose a flexible localization distillation for dense object detection and a selective region distillation based on a new valuable localization region.
We show that 1) logit mimicking can be better than feature imitation;
and 2) the selective region distillation for transferring the classification and localization knowledge is important when distilling object detectors.
%
We hope our method could provide new research intuitions for the object detection community to develop better distillation strategies.
%
In the future, the applications of LD to sparse object detectors (DETR \cite{DETR} series), the heterogeneous detector pairs, and other relevant fields, e.g., instance segmentation, object tracking and 3D object detection, warrant further research.
Besides, since our LD shares the equivalent optimization effect to classification KD, some improved KD methods may also bring gain to LD, e.g., Relational KD \cite{relationKD}, Self-KD \cite{furlanello2018born,zhang2019your}, Teacher Assistant KD \cite{TA}, and Decoupled KD \cite{zhao2022decoupled}, etc.
Cross architecture distillation using recent state-of-the-art classification models~
\cite{wu2022p2t,hou2022conv2former,dai2021coatnet,liu2022convnet,Guo2021187} 
as teachers is also an interesting direction to explore.}



\ifCLASSOPTIONcaptionsoff
  \newpage
\fi

\renewcommand\thesubsectiondis{A.\arabic{subsection}}

\renewcommand{\thetable}{A\arabic{table}}
\renewcommand{\thefigure}{A\arabic{figure}}

\setcounter{table}{0}
\setcounter{proposition}{0}

\begin{appendices}

\section*{appendix}

\subsection{Property of LD}

Some Notations.

$\bm{g}^i=[g_1,g_2,\cdots,g_n]$, where $g_i=1$, and $0$ otherwise.

$\bm{e}=[e_1,e_2,\cdots,e_n]\in\mathbb{R}^n$ is the uniformly discretized variable for the regression range $[e_{\min}, e_{\max}]$.

The gradient of the cross-entropy (CE) loss $w.r.t.$ one of the logit $z_i\in\bm{z}_S$, $i\in\{1,2,\cdots, n\}$ can be represented as:
\begin{equation}\label{eq:ce}
	\frac{\partial\mathcal{L}_\text{CE}}{\partial z_i} = p_i-g_i,
\end{equation}
where $p_i$ is the predicted class probability at location $i$ and $z_S$ is the logit vector produced by the student network.
The gradient of the KD loss along with the CE loss $w.r.t.$ one of the logit $z_i\in\bm{z}_S$ can be represented as:
\begin{equation}\label{eq:cekd}
	\frac{\partial\mathcal{L}^{KD}}{\partial z_i} = \gamma(p_i-g_i) + \frac{\lambda}{\tau}({p_\tau}_i-{q_\tau}_i),
\end{equation}
where $\gamma$ and $\lambda$ are the CE and KD loss weights and $\tau$ is the temperature.
We follow the notations in \cite{tang2020understanding}, and denote $\frac{\partial\mathcal{L}_\text{CE}}{\partial z_i}$ by $\partial_i$ and $\frac{\partial\mathcal{L}^{KD}}{\partial z_i}$ by $\frac{\partial^{KD}_i}{\partial_i}$.
The ratio of Eq. \ref{eq:cekd} and Eq. \ref{eq:ce} indicates that KD performs gradient rescaling to the CE loss in the logits space.

\begin{definition}\label{def}
Let $\bm{p}\in\mathbb{R}^n$ be a predicted probability vector of a network, $M_i>0, i\in\{1,2,\cdots,n\}$ are predefined thresholds. $\bm{p}$ is called $M_i$-well-performed for a task $\mathbf{T}$ if the distance from $\bm{p}$ to its corresponding ground-truth vector $\bm{g}^i$ is bounded by $M_i$.
\end{definition}

\begin{lemma}\label{lemma1}
If two predicted probability vectors $\bm{p},\bm{q}$ are respectively $M_i$-well-performed and $M_j$-well-performed for the integer position classification with ground-truth vectors $\bm{g}^i,\bm{g}^j$, then their linear combination $u_1\bm{p}+u_2\bm{q}$ is $M$-well-performed for the float point number position localization with ground-truth value $y=u_1 e_i+u_2 e_j$, where $M=\max\{M_i,M_j\}$ and $u_1+u_2=1$.
\end{lemma}

\begin{proof}
By Def. \ref{def}, the two distances satisfy $d(\bm{p},\bm{g}^i)\leqslant M_i$, $d(\bm{q},\bm{g}^j)\leqslant M_j$, where $\bm{g}^i\neq\bm{g}^j$. Note that $d(\cdot, \cdot)$ here can be an arbitrary distance metric, e.g., the $l_2$ distance.

A float point number position localization requires a probability, which can be linearly interpolated by $\bm{l}=u_1\bm{p}+u_2\bm{q}$, and its ground-truth vector is $\bm{g}=u_1\bm{g}^i+u_2\bm{g}^j$.

Then we get 
\begin{align}
    d(\bm{l},\bm{g})&=d(u_1\bm{p}+u_2\bm{q},u_1\bm{g}^i+u_2\bm{g}^j)\\
    &\leqslant d(u_1\bm{p}+u_2\bm{q},u_1\bm{g}^i+u_2\bm{q}) \nonumber \\
    &\quad + d(u_1\bm{g}^i+u_2\bm{q},u_1\bm{g}^i+u_2\bm{g}^j)\\
    &=d(u_1\bm{p},u_1\bm{g}^i)+d(u_2\bm{q},u_2\bm{g}^j)\\
    &\leqslant u_1M_i+u_2M_j\\
    &\leqslant \max\{M_i,M_j\}\\
    &=M.
\end{align}

Hence the network is $M$-well-performed for the float point number position localization.
\end{proof}

\begin{lemma}\label{lemma2}
If $\bm{l}$ is a localization probability vector with ground-truth value $y=u_1e_i+u_2e_j$, where $u_1+u_2=1$, then $\bm{l}$ can be decomposed into two classification probabilities $\bm{p}$ and $\bm{q}$ with ground-truth vectors $\bm{g}^i$ and $\bm{g}^j$.
\end{lemma}

\begin{proof}
Let $\bm{l}\in\mathbb{R}^n$ be a predicted localization probability and $\bm{g}$ be its ground-truth vector.
It is easy to decompose $\bm{g}$ into two integer position ground-truth vectors $\bm{g}^i$ and $\bm{g}^j$, satisfying $\bm{g}=u_1\bm{g}^i+u_2\bm{g}^j$.

Existence of the decomposition of $\bm{l}$:

To decompose $\bm{l}$ into two classification probabilities $\bm{p}$ and $\bm{q}$, satisfying $\bm{l}=u_1\bm{p}+u_2\bm{q}$, we solve the following linear equations,
\begin{equation}
AX=\bm{b}\Longleftrightarrow\left\{\begin{array}{ll}
    \sum\limits_ip_i=1,\\
    \sum\limits_iq_i=1,\\ 
    u_1p_1+u_2q_1=l_1,\\
    u_1p_2+u_2q_2=l_2,\\
    \vdots\\
    u_1p_n+u_2q_n=l_n,\\    
\end{array}\right.
\end{equation}
where $X=(p_1,p_2,\cdots,p_n,q_1,q_2,\cdots,q_n)^T$, 
and the augmented matrix $(A,\bm{b})$ is given by:
\begin{equation}
\left(\begin{array}{ccccccccccc}
     1 & 1 & 1 & \cdots & 1 & 0 & 0 & 0 & \cdots & 0 & 1\\
     0 & 0 & 0 & \cdots & 0 & 1 & 1 & 1 & \cdots & 1 & 1\\
     u_1 & 0 & 0 & \cdots & 0 & u_2 & 0 & 0 & \cdots & 0 & l_1\\
     0 & u_1 & 0 & \cdots & 0 & 0 & u_2 & 0 & \cdots & 0 & l_2\\
     0 & 0 & u_1 & \cdots & 0 & 0 & 0 & u_2 & \cdots & 0 & l_3\\
     \vdots & \vdots & \vdots & \ddots & \vdots & \vdots & \vdots & \vdots & \ddots & \vdots & \vdots \\
     0 & 0 & 0 & \cdots & u_1 & 0 & 0 & 0 & \cdots & u_2 & l_n\\
\end{array}\right).
\end{equation}

By applying Elementary Row Operations, the matrix $(A,\bm{b})$ is equivalent to
\begin{equation}
\left(\begin{array}{ccccccccccc}
     1 & 0 & 0 & \cdots & 0 & 0 & -\frac{u_2}{u_1} & -\frac{u_2}{u_1} & \cdots & -\frac{u_2}{u_1} & \frac{l_1-u_2}{u_1} \\
     0 & 1 & 0 & \cdots & 0 & 0 & \frac{u_2}{u_1} & 0 & \cdots & 0 & \frac{l_2}{u_1} \\
     0 & 0 & 1 & \cdots & 0 & 0 & 0 & \frac{u_2}{u_1} & \cdots & 0 & \frac{l_3}{u_1} \\
     \vdots & \vdots & \vdots & \ddots & \vdots & \vdots & \vdots & \vdots & \ddots & \vdots & \vdots \\
     0 & 0 & 0 & \cdots & 1 & 0 & 0 & 0 & \cdots & \frac{u_2}{u_1} & \frac{l_n}{u_1} \\
     0 & 0 & 0 & \cdots & 0 & 1 & 1 & 1 & \cdots & 1 & 1 \\
     0 & 0 & 0 & \cdots & 0 & 0 & 0 & 0 & \cdots & 0 & 0 \\
\end{array}\right),
\end{equation}
and we obtain the rank of the coefficient matrix $A$ is equal to the rank of the augmented matrix $(A,\bm{b})$, which is $n+1$.
Note that $n+1<2n$ when $n>1$.
Thus the above linear equations have infinite solutions.
\end{proof}

The following proposition describes the relation between KD and LD that conducting LD to optimize a localization probability is equivalent to conducting KD to optimize two classification probabilities.
\begin{theorem}
\label{appendix:1}
Let $\bm{s}$ be the student's predicted probability vector, $u_1$ and $u_2$ are two constants and their summation is $1$. We have
\begin{enumerate}
    \item If $\bm{p}$ and $\bm{q}$ are two classification probabilities, LD effect on the linear combination $\bm{l}=u_1\bm{p}+u_2\bm{q}$ is equal to the linear combination of KD effects on $\bm{p},\bm{q}$.
    
    \item If $\bm{l}$ is a localization probability, LD effect on $\bm{l}$ is equal to two KD effects on its decomposition $\bm{p}$ and $\bm{q}$.
\end{enumerate}
\end{theorem}

\begin{proof}
We first denote the derivatives of the KD loss of two probabilities $\bm{s},\bm{p}$ $w.r.t.$ a given logit $z_i$ by $\partial KD^p_i$, and $\partial LD^p_i$ likewise for the LD loss.

1) According to Lemma \ref{lemma1}, the linear combination $\bm{l}=u_1\bm{p}+u_2\bm{q}$ is well defined and the derivatives of the LD loss of $\bm{s}, \bm{l}$ $w.r.t.$ a given logit $z_i$ is given by:
\begin{align}
    \partial LD^l_i&={s_\tau}_i-{l_\tau}_i\\
    &=u_1{s_\tau}_i+u_2{s_\tau}_i-(u_1{p_\tau}_i+u_2{q_\tau}_i)\\
    &=u_1({s_\tau}_i-{p_\tau}_i)+u_2({s_\tau}_i-{q_\tau}_i)\\
    &=u_1\partial KD^p_i+u_2\partial KD^q_i \label{eq:KDLDapp}
\end{align}

2) According to Lemma \ref{lemma2}, the decomposition of $\bm{l}$ exists, which is written as $\bm{l}=u_1\bm{p}+u_2\bm{q}$. Then Eq. \ref{eq:KDLDapp} still holds.
\end{proof}

\subsection{Gradient Rescaling}\label{appendix:2}
We first give the lemma in \cite{tang2020understanding}.
\begin{lemma}\label{lemma}
Let ${q_\tau}_t={p_\tau}_t+c_t+\eta$, where $c_t$ is teacher’s relative prediction confidence on the ground-truth class $t$ and $\eta$ is a zero-mean random noise. Then the logit's gradient rescaling factor by applying KD is given by:
\begin{equation}
    \mathbb{E}_\eta\left[\frac{\partial^{KD}_t}{\partial_t}\right]
    =\mathbb{E}_\eta\left[\frac{\sum\limits_{i\neq t}\partial^{KD}_i}{\sum\limits_{i\neq t}\partial_i}\right]
    =\gamma+\frac{\lambda}{\tau}\left(\frac{c_t}{1-p_t}\right).
\end{equation}
\end{lemma}

Next, we give the corollary of Lemma \ref{lemma}, which shows that LD performs gradient rescaling to distribution focal loss (DFL) \cite{gfocal} in the logits space.

\begin{corollary}
Let ${q_\tau}_i={p_\tau}_i+c_i+\eta_i$, where $c_i$ is teacher’s relative prediction confidence at position $i$, $\eta_i$ is a zero-mean random noise. Then the logit's gradient rescaling factor to DFL by applying LD is given by:
\begin{equation}\label{eq:rescaling1}
    \mathbb{E}_\eta\left[\frac{\partial^{LD}_i}{\partial_i}\right]
    =\mathbb{E}_\eta\left[\frac{\sum\limits_{s\neq i}\partial^{LD}_s}{\sum\limits_{s\neq i}\partial_s}\right]
    =\gamma +\frac{\lambda}{\tau}\left(\frac{c_i}{u_i-p_i}\right),
\end{equation}
where $\frac{\partial^{LD}_i}{\partial_i}$ denotes the gradient of LD loss along with DFL w.r.t. logits $z_i$, $u_i$ and $u_j$ are two constants and their summation is 1, $\gamma$ and $\lambda$ are the loss weights of the DFL and LD loss respectively, and $\tau$ is temperature.
\end{corollary}
\begin{proof}
Following \cite{gfocal}, DFL is defined as the linear combination of two CE loss at position $i$ and $j$,
\begin{equation}\label{eq:dfl}
    \mathcal{L}_{DFL}=u_i\mathcal{H}(\bm{p},\bm{g^i})+u_j\mathcal{H}(\bm{p},\bm{g^j}),
\end{equation}
where $\bm{g^i}=\{0,1\}^n$ are ground-truth labels whose value is 1 at position $i$ and 0 otherwise.
One can easily get the gradient of DFL $w.r.t.$ the logit $z_i$,
\begin{equation}\label{eq:dflgradient}
    \frac{\partial\mathcal{L}_\text{DFL}}{\partial z_i} = u_i(p_i-g_i)+u_jp_i=p_i-u_i,
\end{equation}
and we still use the notation $\partial_i$ to represent $\frac{\partial\mathcal{L}_\text{DFL}}{\partial z_i}$.
With LD, the total loss is given by:
\begin{equation}\label{eq:celd}
    \mathcal{L}^{LD}=\gamma(u_i\mathcal{H}(\bm{p},\bm{g^i})+u_j\mathcal{H}(\bm{p},\bm{g^j}))+\lambda\mathcal{H}(\bm{p}_\tau,\bm{q}_\tau),
\end{equation}
The gradient of LD loss along with DFL $w.r.t.$ the logit $z_i\in\bm{z}_S$ can be represented as:
\begin{equation}\label{eq:ldgradient}
	\frac{\partial\mathcal{L}^{LD}}{\partial z_i} = \gamma u_i(p_i-g_i) + \gamma u_jp_i + \frac{\lambda}{\tau}({p_\tau}_i-{q_\tau}_i),
\end{equation}
and we still denote $\partial^{LD}_i=\frac{\partial\mathcal{L}_\text{LD}}{\partial z_i}$.
According to Lemma \ref{lemma}, we have the ratio of Eq. \ref{eq:ldgradient} and Eq. \ref{eq:dflgradient},
\begin{align}
	\mathbb{E}_\eta\left[\frac{\partial^{LD}_i}{\partial_i}\right]
	&=\gamma u_i\frac{p_i-g_i}{p_i-u_i}+\gamma \frac{u_jp_i}{p_i-u_i}-\frac{\lambda}{\tau}\frac{c_i}{p_i-u_i}\\
	&=\gamma + \frac{\lambda}{\tau}\frac{c_i}{u_i-p_i}.
\end{align}
Thus, the sum of the incorrect position gradients is given by:
\begin{equation}
\begin{array}{ll}
    \sum\limits_{s\neq i}\partial^{LD}_s\\
	=\gamma u_i\sum\limits_{s\neq i}p_s + \gamma u_j\sum\limits_{s\neq i,j}p_s 
	+ \gamma u_j(p_j-g_j) + \frac{\lambda}{\tau}\sum\limits_{s\neq i}({p_\tau}_s-{q_\tau}_s)\\
	=\gamma u_i(g_i-p_i) + \gamma u_j(g_i-p_i)-\gamma u_jg_j 
	+ \frac{\lambda}{\tau}({q_\tau}_s-{p_\tau}_s)\\
	=\gamma u_i(g_i-p_i) - \gamma u_jp_i + \frac{\lambda}{\tau}({q_\tau}_s-{p_\tau}_s)\\
	=-\partial^{LD}_i.\\
\end{array}
\end{equation}
Similarly applies for $\partial_s$, and hence the proof.
\end{proof}

\end{appendices}

{\small
	\bibliographystyle{IEEEtran}
	\bibliography{egbib}

\begin{thebibliography}{10}
\providecommand{\url}[1]{#1}
\csname url@samestyle\endcsname
\providecommand{\newblock}{\relax}
\providecommand{\bibinfo}[2]{#2}
\providecommand{\BIBentrySTDinterwordspacing}{\spaceskip=0pt\relax}
\providecommand{\BIBentryALTinterwordstretchfactor}{4}
\providecommand{\BIBentryALTinterwordspacing}{\spaceskip=\fontdimen2\font plus
\BIBentryALTinterwordstretchfactor\fontdimen3\font minus
  \fontdimen4\font\relax}
\providecommand{\BIBforeignlanguage}[2]{{%
\expandafter\ifx\csname l@#1\endcsname\relax
\typeout{** WARNING: IEEEtran.bst: No hyphenation pattern has been}%
\typeout{** loaded for the language `#1'. Using the pattern for}%
\typeout{** the default language instead.}%
\else
\language=\csname l@#1\endcsname
\fi
#2}}
\providecommand{\BIBdecl}{\relax}
\BIBdecl

\bibitem{hinton2015distilling}
G.~Hinton, O.~Vinyals, and J.~Dean, ``Distilling the knowledge in a neural
  network,'' \emph{arXiv preprint arXiv:1503.02531}, 2015.

\bibitem{FitNets}
A.~Romero, N.~Ballas, S.~E. Kahou, A.~Chassang, C.~Gatta, and Y.~Bengio,
  ``Fitnets: Hints for thin deep nets,'' in \emph{Int. Conf. Learn.
  Represent.}, 2015.

\bibitem{zagoruyko2016paying}
S.~Zagoruyko and N.~Komodakis, ``Paying more attention to attention: Improving
  the performance of convolutional neural networks via attention transfer,'' in
  \emph{Int. Conf. Learn. Represent.}, 2017.

\bibitem{kim2018paraphrasing}
J.~Kim, S.~Park, and N.~Kwak, ``Paraphrasing complex network: network
  compression via factor transfer,'' in \emph{Adv. Neural Inform. Process.
  Syst.}, 2018, pp. 2765--2774.

\bibitem{Jin_2019_ICCV}
X.~Jin, B.~Peng, Y.~Wu, Y.~Liu, J.~Liu, D.~Liang, J.~Yan, and X.~Hu,
  ``Knowledge distillation via route constrained optimization,'' in \emph{Int.
  Conf. Comput. Vis.}, 2019.

\bibitem{wang2021distilling}
G.-H. Wang, Y.~Ge, and J.~Wu, ``Distilling knowledge by mimicking features,''
  \emph{IEEE Trans. Pattern Anal. Mach. Intell.}, 2021.

\bibitem{chen2017learning}
G.~Chen, W.~Choi, X.~Yu, T.~Han, and M.~Chandraker, ``Learning efficient object
  detection models with knowledge distillation,'' in \emph{Adv. Neural Inform.
  Process. Syst.}, 2017.

\bibitem{sun2020distilling}
R.~Sun, F.~Tang, X.~Zhang, H.~Xiong, and Q.~Tian, ``Distilling object detectors
  with task adaptive regularization,'' \emph{arXiv preprint arXiv:2006.13108},
  2020.

\bibitem{wang2019distilling}
T.~Wang, L.~Yuan, X.~Zhang, and J.~Feng, ``Distilling object detectors with
  fine-grained feature imitation,'' in \emph{IEEE Conf. Comput. Vis. Pattern
  Recog.}, 2019.

\bibitem{zhang2020improve}
L.~Zhang and K.~Ma, ``Improve object detection with feature-based knowledge
  distillation: Towards accurate and efficient detectors,'' in \emph{Int. Conf.
  Learn. Represent.}, 2020.

\bibitem{kang2021instanceconditional}
Z.~Kang, P.~Zhang, X.~Zhang, J.~Sun, and N.~Zheng, ``Instance-conditional
  knowledge distillation for object detection,'' in \emph{Adv. Neural Inform.
  Process. Syst.}, 2021.

\bibitem{gfocal}
X.~Li, W.~Wang, L.~Wu, S.~Chen, X.~Hu, J.~Li, J.~Tang, and J.~Yang,
  ``{Generalized Focal Loss:} learning qualified and distributed bounding boxes
  for dense object detection,'' in \emph{Adv. Neural Inform. Process. Syst.},
  2020.

\bibitem{offsetbin}
H.~Qiu, H.~Li, Q.~Wu, and H.~Shi, ``Offset bin classification network for
  accurate object detection,'' in \emph{IEEE Conf. Comput. Vis. Pattern
  Recog.}, 2020.

\bibitem{yolov1}
J.~Redmon, S.~Divvala, R.~Girshick, and A.~Farhadi, ``You only look once:
  Unified, real-time object detection,'' in \emph{IEEE Conf. Comput. Vis.
  Pattern Recog.}, 2016.

\bibitem{SSD}
W.~Liu, D.~Anguelov, D.~Erhan, C.~Szegedy, S.~Reed, C.-Y. Fu, and A.~C. Berg,
  ``Ssd: Single shot multibox detector,'' in \emph{Eur. Conf. Comput. Vis.},
  2016.

\bibitem{fasterrcnn}
S.~Ren, K.~He, R.~Girshick, and J.~Sun, ``Faster {R-CNN}: Towards real-time
  object detection with region proposal networks,'' in \emph{Adv. Neural
  Inform. Process. Syst.}, 2015.

\bibitem{FCOS}
Z.~Tian, C.~Shen, H.~Chen, and T.~He, ``{FCOS}: Fully convolutional one-stage
  object detection,'' in \emph{Int. Conf. Comput. Vis.}, 2019.

\bibitem{PyTorch}
A.~Paszke, S.~Gross, S.~Chintala, G.~Chanan, E.~Yang, Z.~Devito, Z.~Lin,
  A.~Desmaison, L.~Antiga, and A.~Lerer, ``Automatic differentiation in
  pytorch,'' in \emph{Adv. Neural Inform. Process. Syst.}, 2017.

\bibitem{hu2020jittor}
S.-M. Hu, D.~Liang, G.-Y. Yang, G.-W. Yang, and W.-Y. Zhou, ``Jittor: A novel
  deep learning framework with meta-operators and unified graph execution,''
  \emph{Science China Information Sciences}, vol.~63, no. 222103, pp. 1--21,
  2020.

\bibitem{zheng2022LD}
Z.~Zheng, R.~Ye, P.~Wang, D.~Ren, W.~Zuo, Q.~Hou, and M.~Cheng, ``Localization
  distillation for dense object detection,'' in \emph{IEEE Conf. Comput. Vis.
  Pattern Recog.}, 2022.

\bibitem{Zagoruyko2017AT}
S.~Zagoruyko and N.~Komodakis, ``Paying more attention to attention: Improving
  the performance of convolutional neural networks via attention transfer,'' in
  \emph{Int. Conf. Learn. Represent.}, 2017.

\bibitem{bae2020densely}
J.-H. Bae, D.~Yeo, J.~Yim, N.-S. Kim, C.-S. Pyo, and J.~Kim, ``Densely
  distilled flow-based knowledge transfer in teacher-student framework for
  image classification,'' \emph{IEEE Transactions on Image Processing},
  vol.~29, pp. 5698--5710, 2020.

\bibitem{relationKD}
W.~Park, D.~Kim, Y.~Lu, and M.~Cho, ``Relational knowledge distillation,'' in
  \emph{IEEE Conf. Comput. Vis. Pattern Recog.}, 2019.

\bibitem{TA}
S.~I. Mirzadeh, M.~Farajtabar, A.~Li, N.~Levine, A.~Matsukawa, and
  H.~Ghasemzadeh, ``Improved knowledge distillation via teacher assistant,'' in
  \emph{Association for the Advancement of Artificial Intelligence}, 2020.

\bibitem{DenselyTA}
W.~Son, J.~Na, J.~Choi, and W.~Hwang, ``Densely guided knowledge distillation
  using multiple teacher assistants,'' in \emph{Int. Conf. Comput. Vis.}, 2021.

\bibitem{Li_2017_CVPR}
Q.~Li, S.~Jin, and J.~Yan, ``Mimicking very efficient network for object
  detection,'' in \emph{IEEE Conf. Comput. Vis. Pattern Recog.}, 2017.

\bibitem{GIbox}
X.~Dai, Z.~Jiang, Z.~Wu, Y.~Bao, Z.~Wang, S.~Liu, and E.~Zhou, ``General
  instance distillation for object detection,'' in \emph{IEEE Conf. Comput.
  Vis. Pattern Recog.}, 2021.

\bibitem{defeat}
J.~Guo, K.~Han, Y.~Wang, H.~Wu, X.~Chen, C.~Xu, and C.~Xu, ``Distilling object
  detectors via decoupled features,'' in \emph{IEEE Conf. Comput. Vis. Pattern
  Recog.}, 2021.

\bibitem{FeatureRichness}
D.~Zhixing, R.~Zhang, M.~Chang, S.~Liu, T.~Chen, Y.~Chen \emph{et~al.},
  ``Distilling object detectors with feature richness,'' in \emph{Adv. Neural
  Inform. Process. Syst.}, 2021.

\bibitem{li2021knowledge}
G.~Li, X.~Li, Y.~Wang, S.~Zhang, Y.~Wu, and D.~Liang, ``Knowledge distillation
  for object detection via rank mimicking and prediction-guided feature
  imitation,'' in \emph{Association for the Advancement of Artificial
  Intelligence}, 2022.

\bibitem{locnet}
S.~Gidaris and N.~Komodakis, ``Locnet: Improving localization accuracy for
  object detection,'' in \emph{IEEE Conf. Comput. Vis. Pattern Recog.}, 2016.

\bibitem{wang2019region}
J.~Wang, K.~Chen, S.~Yang, C.~C. Loy, and D.~Lin, ``Region proposal by guided
  anchoring,'' in \emph{IEEE Conf. Comput. Vis. Pattern Recog.}, 2019.

\bibitem{SABL}
J.~Wang, W.~Zhang, Y.~Cao, K.~Chen, J.~Pang, T.~Gong, J.~Shi, C.~C. Loy, and
  D.~Lin, ``Side-aware boundary localization for more precise object
  detection,'' in \emph{Eur. Conf. Comput. Vis.}, 2020.

\bibitem{zhu2019feature}
C.~Zhu, Y.~He, and M.~Savvides, ``Feature selective anchor-free module for
  single-shot object detection,'' in \emph{IEEE Conf. Comput. Vis. Pattern
  Recog.}, 2019.

\bibitem{gridrcnn}
X.~Lu, B.~Li, Y.~Yue, Q.~Li, and J.~Yan, ``Grid {R-CNN},'' in \emph{IEEE Conf.
  Comput. Vis. Pattern Recog.}, 2019.

\bibitem{kong2018deep}
T.~Kong, F.~Sun, C.~Tan, H.~Liu, and W.~Huang, ``Deep feature pyramid
  reconfiguration for object detection,'' in \emph{Eur. Conf. Comput. Vis.},
  2018.

\bibitem{SCRDet}
X.~Yang, J.~Yang, J.~Yan, Y.~Zhang, T.~Zhang, Z.~Guo, X.~Sun, and K.~Fu,
  ``Scrdet: Towards more robust detection for small, cluttered and rotated
  objects,'' in \emph{Int. Conf. Comput. Vis.}, 2019.

\bibitem{GWD}
X.~Yang, J.~Yan, Q.~Ming, W.~Wang, X.~Zhang, and Q.~Tian, ``Rethinking rotated
  object detection with gaussian wasserstein distance loss,'' in
  \emph{International Conference on Machine Learning (ICML)}, 2021.

\bibitem{KLD}
X.~Yang, X.~Yang, J.~Yang, Q.~Ming, W.~Wang, Q.~Tian, and J.~Yan, ``Learning
  high-precision bounding box for rotated object detection via kullback-leibler
  divergence,'' in \emph{Adv. Neural Inform. Process. Syst.}, 2021.

\bibitem{VFNet}
H.~Zhang, Y.~Wang, F.~Dayoub, and N.~S{\"u}nderhauf, ``Varifocalnet: An
  iou-aware dense object detector,'' in \emph{IEEE Conf. Comput. Vis. Pattern
  Recog.}, 2021.

\bibitem{felzenszwalb2009object}
P.~F. Felzenszwalb, R.~B. Girshick, D.~McAllester, and D.~Ramanan, ``Object
  detection with discriminatively trained part-based models,'' \emph{IEEE
  Trans. Pattern Anal. Mach. Intell.}, vol.~32, no.~9, pp. 1627--1645, 2009.

\bibitem{Han2019221}
L.~Han, P.~Tao, and R.~R. Martin, ``Livestock detection in aerial images using
  a fully convolutional network,'' \emph{Computational Visual Media}, vol.~5,
  no.~2, p. 221 – 228, 2019.

\bibitem{cascadercnn}
Z.~Cai and N.~Vasconcelos, ``Cascade {R-CNN}: Delving into high quality object
  detection,'' in \emph{IEEE Conf. Comput. Vis. Pattern Recog.}, 2018.

\bibitem{librarcnn}
J.~Pang, K.~Chen, J.~Shi, H.~Feng, W.~Ouyang, and D.~Lin, ``Libra {R-CNN}:
  Towards balanced learning for object detection,'' in \emph{IEEE Conf. Comput.
  Vis. Pattern Recog.}, 2019.

\bibitem{DynamicRCNN}
H.~Zhang, H.~Chang, B.~Ma, N.~Wang, and X.~Chen, ``Dynamic {R-CNN}: Towards
  high quality object detection via dynamic training,'' in \emph{Eur. Conf.
  Comput. Vis.}, 2020.

\bibitem{yolov2}
J.~Redmon and A.~Farhadi, ``Yolo9000: better, faster, stronger,'' in \emph{IEEE
  Conf. Comput. Vis. Pattern Recog.}, 2017.

\bibitem{yolov3}
------, ``Yolov3: An incremental improvement,'' \emph{arXiv preprint
  arXiv:1804.02767}, 2018.

\bibitem{yolov4}
A.~Bochkovskiy, C.-Y. Wang, and H.-Y.~M. Liao, ``Yolov4: Optimal speed and
  accuracy of object detection,'' \emph{arXiv preprint arXiv:2004.10934}, 2020.

\bibitem{DSSD}
C.-Y. Fu, W.~Liu, A.~Ranga, A.~Tyagi, and A.~C. Berg, ``{DSSD}: Deconvolutional
  single shot detector,'' \emph{arXiv:1701.06659}, 2017.

\bibitem{STDN}
P.~Zhou, B.~Ni, C.~Geng, J.~Hu, and Y.~Xu, ``Scale-transferrable object
  detection,'' in \emph{IEEE Conf. Comput. Vis. Pattern Recog.}, 2018.

\bibitem{unitbox}
J.~Yu, Y.~Jiang, Z.~Wang, Z.~Cao, and T.~Huang, ``{UnitBox:} an advanced object
  detection network,'' in \emph{ACM Int. Conf. Multimedia}, 2016.

\bibitem{giou}
H.~Rezatofighi, N.~Tsoi, J.~Gwak, A.~Sadeghian, I.~Reid, and S.~Savarese,
  ``{Generalized Intersection over Union}: A metric and a loss for bounding box
  regression,'' in \emph{IEEE Conf. Comput. Vis. Pattern Recog.}, 2019.

\bibitem{diou}
Z.~Zheng, P.~Wang, W.~Liu, J.~Li, R.~Ye, and D.~Ren, ``{Distance-IoU Loss}:
  Faster and better learning for bounding box regression,'' in
  \emph{Association for the Advancement of Artificial Intelligence}, 2020.

\bibitem{ciou}
Z.~Zheng, P.~Wang, D.~Ren, W.~Liu, R.~Ye, Q.~Hu, and W.~Zuo, ``Enhancing
  geometric factors in model learning and inference for object detection and
  instance segmentation,'' \emph{IEEE Transactions on Cybernetics}, 2021.

\bibitem{softernms}
Y.~He, C.~Zhu, J.~Wang, M.~Savvides, and X.~Zhang, ``Bounding box regression
  with uncertainty for accurate object detection,'' in \emph{IEEE Conf. Comput.
  Vis. Pattern Recog.}, 2019.

\bibitem{gaussian_yolov3}
J.~Choi, D.~Chun, H.~Kim, and H.-J. Lee, ``Gaussian {YOLOv3}: An accurate and
  fast object detector using localization uncertainty for autonomous driving,''
  in \emph{Int. Conf. Comput. Vis.}, 2019.

\bibitem{gfocalv2}
X.~Li, W.~Wang, X.~Hu, J.~Li, J.~Tang, and J.~Yang, ``Generalized focal loss
  v2: Learning reliable localization quality estimation for dense object
  detection,'' in \emph{IEEE Conf. Comput. Vis. Pattern Recog.}, 2021.

\bibitem{iounet}
B.~Jiang, R.~Luo, J.~Mao, T.~Xiao, and Y.~Jiang, ``Acquisition of localization
  confidence for accurate object detection,'' in \emph{Eur. Conf. Comput.
  Vis.}, 2018.

\bibitem{mask_scoring}
Z.~Huang, L.~Huang, Y.~Gong, C.~Huang, and X.~Wang, ``Mask scoring {R-CNN},''
  in \emph{IEEE Conf. Comput. Vis. Pattern Recog.}, 2019.

\bibitem{polarmask}
E.~Xie, P.~Sun, X.~Song, W.~Wang, X.~Liu, D.~Liang, C.~Shen, and P.~Luo,
  ``Polarmask: Single shot instance segmentation with polar representation,''
  in \emph{IEEE Conf. Comput. Vis. Pattern Recog.}, 2020.

\bibitem{zhou2022mmrotate}
Y.~Zhou, X.~Yang, G.~Zhang, J.~Wang, Y.~Liu, L.~Hou, X.~Jiang, X.~Liu, J.~Yan,
  C.~Lyu, W.~Zhang, and K.~Chen, ``Mmrotate: A rotated object detection
  benchmark using pytorch,'' in \emph{ACM Int. Conf. Multimedia}, 2022.

\bibitem{RRPN}
J.~Ma, W.~Shao, H.~Ye, L.~Wang, H.~Wang, Y.~Zheng, and X.~Xue,
  ``Arbitrary-oriented scene text detection via rotation proposals,''
  \emph{IEEE Transactions on Multimedia}, vol.~20, no.~11, pp. 3111--3122,
  2018.

\bibitem{lin2017focal}
T.-Y. Lin, P.~Goyal, R.~Girshick, K.~He, and P.~Doll{\'a}r, ``Focal loss for
  dense object detection,'' in \emph{Int. Conf. Comput. Vis.}, 2017.

\bibitem{RSDet}
W.~Qian, X.~Yang, S.~Peng, Y.~Guo, and J.~Yan, ``Learning modulated loss for
  rotated object detection,'' in \emph{Association for the Advancement of
  Artificial Intelligence}, 2021.

\bibitem{CSL}
X.~Yang and J.~Yan, ``Arbitrary-oriented object detection with circular smooth
  label,'' in \emph{Eur. Conf. Comput. Vis.}, 2020.

\bibitem{tensorflow}
M.~Abadi, P.~Barham, J.~Chen, Z.~Chen, A.~Davis, J.~Dean, M.~Devin,
  S.~Ghemawat, G.~Irving, M.~Isard \emph{et~al.}, ``$\{$TensorFlow$\}$: A
  system for $\{$Large-Scale$\}$ machine learning,'' in \emph{12th USENIX
  symposium on operating systems design and implementation (OSDI 16)}, 2016,
  pp. 265--283.

\bibitem{piou}
Z.~Chen, K.~Chen, W.~Lin, J.~See, H.~Yu, Y.~Ke, and C.~Yang, ``Piou loss:
  Towards accurate oriented object detection in complex environments,'' in
  \emph{Eur. Conf. Comput. Vis.}, 2020.

\bibitem{kfiou}
X.~Yang, Y.~Zhou, G.~Zhang, J.~Yang, W.~Wang, J.~Yan, X.~Zhang, and Q.~Tian,
  ``The kfiou loss for rotated object detection,'' \emph{arXiv preprint
  arXiv:2201.12558}, 2022.

\bibitem{r3det}
X.~Yang, Q.~Liu, J.~Yan, A.~Li, Z.~Zhang, and G.~Yu, ``R3det: Refined
  single-stage detector with feature refinement for rotating object,'' in
  \emph{Association for the Advancement of Artificial Intelligence}, 2021.

\bibitem{tang2020understanding}
J.~Tang, R.~Shivanna, Z.~Zhao, D.~Lin, A.~Singh, E.~H. Chi, and S.~Jain,
  ``Understanding and improving knowledge distillation,'' \emph{arXiv preprint
  arXiv:2002.03532}, 2020.

\bibitem{coco}
T.-Y. Lin, M.~Maire, S.~Belongie, J.~Hays, P.~Perona, D.~Ramanan,
  P.~Doll{\'a}r, and C.~L. Zitnick, ``Microsoft coco: Common objects in
  context,'' in \emph{Eur. Conf. Comput. Vis.}, 2014.

\bibitem{voc}
M.~Everingham, L.~Van~Gool, C.~K.~I. Williams, J.~Winn, and A.~Zisserman, ``The
  pascal visual object classes (voc) challenge,'' \emph{International Journal
  of Computer Vision}, vol.~88, no.~2, pp. 303--338, 2010.

\bibitem{dota}
G.-S. Xia, X.~Bai, J.~Ding, Z.~Zhu, S.~Belongie, J.~Luo, M.~Datcu, M.~Pelillo,
  and L.~Zhang, ``Dota: A large-scale dataset for object detection in aerial
  images,'' in \emph{IEEE Conf. Comput. Vis. Pattern Recog.}, 2018, pp.
  3974--3983.

\bibitem{mmdetection}
K.~Chen, J.~Wang, J.~Pang, Y.~Cao, Y.~Xiong, X.~Li, S.~Sun, W.~Feng, Z.~Liu,
  J.~Xu, Z.~Zhang, D.~Cheng, C.~Zhu, T.~Cheng, Q.~Zhao, B.~Li, X.~Lu, R.~Zhu,
  Y.~Wu, J.~Dai, J.~Wang, J.~Shi, W.~Ouyang, C.~C. Loy, and D.~Lin,
  ``{MMDetection}: Open mmlab detection toolbox and benchmark,'' \emph{arXiv
  preprint arXiv:1906.07155}, 2019.

\bibitem{ResNet}
K.~He, X.~Zhang, S.~Ren, and J.~Sun, ``Deep residual learning for image
  recognition,'' in \emph{IEEE Conf. Comput. Vis. Pattern Recog.}, 2016.

\bibitem{FPN}
T.-Y. Lin, P.~Doll{\'a}r, R.~Girshick, K.~He, B.~Hariharan, and S.~Belongie,
  ``Feature pyramid networks for object detection,'' in \emph{IEEE Conf.
  Comput. Vis. Pattern Recog.}, 2017.

\bibitem{ATSS}
S.~Zhang, C.~Chi, Y.~Yao, Z.~Lei, and S.~Z. Li, ``Bridging the gap between
  anchor-based and anchor-free detection via adaptive training sample
  selection,'' in \emph{IEEE Conf. Comput. Vis. Pattern Recog.}, 2020.

\bibitem{losslandscape}
H.~Li, Z.~Xu, G.~Taylor, C.~Studer, and T.~Goldstein, ``Visualizing the loss
  landscape of neural nets,'' \emph{Adv. Neural Inform. Process. Syst.}, 2018.

\bibitem{xie2017aggregated}
S.~Xie, R.~Girshick, P.~Doll{\'a}r, Z.~Tu, and K.~He, ``Aggregated residual
  transformations for deep neural networks,'' in \emph{IEEE Conf. Comput. Vis.
  Pattern Recog.}, 2017.

\bibitem{DCNv2}
X.~Zhu, H.~Hu, S.~Lin, and J.~Dai, ``Deformable convnets v2: More deformable,
  better results,'' in \emph{IEEE Conf. Comput. Vis. Pattern Recog.}, 2019.

\bibitem{gao2019res2net}
S.-H. Gao, M.-M. Cheng, K.~Zhao, X.-Y. Zhang, M.-H. Yang, and P.~Torr,
  ``Res2net: A new multi-scale backbone architecture,'' \emph{IEEE Trans.
  Pattern Anal. Mach. Intell.}, vol.~43, no.~2, pp. 652--662, 2021.

\bibitem{zhu2019soft}
C.~Zhu, F.~Chen, Z.~Shen, and M.~Savvides, ``Soft anchor-point object
  detection,'' in \emph{IEEE Conf. Comput. Vis. Pattern Recog.}, 2020.

\bibitem{qiu2020borderdet}
H.~Qiu, Y.~Ma, Z.~Li, S.~Liu, and J.~Sun, ``Borderdet: Border feature for dense
  object detection,'' in \emph{Eur. Conf. Comput. Vis.}, 2020.

\bibitem{zhu2020autoassign}
B.~Zhu, J.~Wang, Z.~Jiang, F.~Zong, S.~Liu, Z.~Li, and J.~Sun, ``Autoassign:
  Differentiable label assignment for dense object detection,'' \emph{arXiv
  preprint arXiv:2007.03496}, 2020.

\bibitem{kim2020probabilistic}
K.~Kim and H.~S. Lee, ``Probabilistic anchor assignment with iou prediction for
  object detection,'' in \emph{Eur. Conf. Comput. Vis.}, 2020.

\bibitem{OTA}
Z.~Ge, S.~Liu, Z.~Li, O.~Yoshie, and J.~Sun, ``{OTA}: Optimal transport
  assignment for object detection,'' in \emph{IEEE Conf. Comput. Vis. Pattern
  Recog.}, 2021.

\bibitem{DETR}
N.~Carion, F.~Massa, G.~Synnaeve, N.~Usunier, A.~Kirillov, and S.~Zagoruyko,
  ``End-to-end object detection with transformers,'' in \emph{Eur. Conf.
  Comput. Vis.}, 2020.

\bibitem{furlanello2018born}
T.~Furlanello, Z.~Lipton, M.~Tschannen, L.~Itti, and A.~Anandkumar, ``Born
  again neural networks,'' in \emph{International Conference on Machine
  Learning (ICML)}, 2018, pp. 1607--1616.

\bibitem{zhang2019your}
L.~Zhang, J.~Song, A.~Gao, J.~Chen, C.~Bao, and K.~Ma, ``Be your own teacher:
  Improve the performance of convolutional neural networks via self
  distillation,'' in \emph{Int. Conf. Comput. Vis.}, 2019, pp. 3713--3722.

\bibitem{zhao2022decoupled}
B.~Zhao, Q.~Cui, R.~Song, Y.~Qiu, and J.~Liang, ``Decoupled knowledge
  distillation,'' in \emph{IEEE Conf. Comput. Vis. Pattern Recog.}, 2022.

\bibitem{wu2022p2t}
Y.-H. Wu, Y.~Liu, X.~Zhan, and M.-M. Cheng, ``{P2T}: Pyramid pooling
  transformer for scene understanding,'' \emph{IEEE Trans. Pattern Anal. Mach.
  Intell.}, 2022.

\bibitem{hou2022conv2former}
Q.~Hou, C.-Z. Lu, M.-M. Cheng, and J.~Feng, ``Conv2former: A simple
  transformer-style convnet for visual recognition,'' \emph{arXiv preprint
  arXiv:2211.11943}, 2022.

\bibitem{dai2021coatnet}
Z.~Dai, H.~Liu, Q.~Le, and M.~Tan, ``Coatnet: Marrying convolution and
  attention for all data sizes,'' \emph{Adv. Neural Inform. Process. Syst.},
  vol.~34, 2021.

\bibitem{liu2022convnet}
Z.~Liu, H.~Mao, C.-Y. Wu, C.~Feichtenhofer, T.~Darrell, and S.~Xie, ``A convnet
  for the 2020s,'' in \emph{IEEE Conf. Comput. Vis. Pattern Recog.}, 2022.

\bibitem{Guo2021187}
M.-H. Guo, J.-X. Cai, Z.-N. Liu, T.-J. Mu, R.~R. Martin, and S.-M. Hu, ``Pct:
  Point cloud transformer,'' \emph{Computational Visual Media}, vol.~7, no.~2,
  p. 187 – 199, 2021.

\end{thebibliography}
}


\newcommand{\addPhoto}[1]{\includegraphics[width=1in,height=1.15in,clip,keepaspectratio]{pic/photo/#1}}

\vspace{-40pt}

\begin{IEEEbiography}[\addPhoto{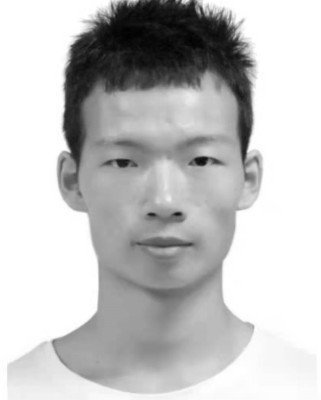}]{Zhaohui Zheng}
	received the M.S. degree in computational mathematics from 
  Tianjin University in 2021.
	He is currently a Ph.D. candidate with the School of Computer Science 
  at Nankai University, Tianjin, China.
	His research interests include object detection, 
  instance segmentation and knowledge distillation.
\end{IEEEbiography}

\vspace{-35pt}
\begin{IEEEbiography}[\addPhoto{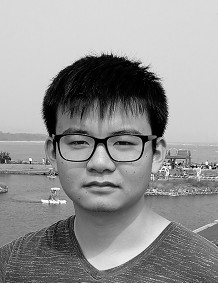}]{Rongguang Ye}
  received the B.S. and M.S. degrees from the School of Mathematics, 
  Tianjin University, Tianjin, China, in 2019 and 2022.
  He is now working at Intel Asia-Pacific Research And Development Ltd 
  as an AI framework engineer.
  His research interests include object detection and computer vision.
\end{IEEEbiography}

\vspace{-35pt}
\begin{IEEEbiography}[\addPhoto{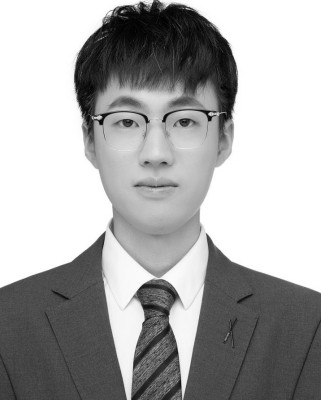}]{Qibin Hou} 
  received his Ph.D. degree from the School of Computer Science, 
  Nankai University. 
  Then, he worked at the National University of Singapore as a research fellow. 
  Now, he is an associate professor at School of Computer Science, 
  Nankai University. 
  He has published more than 30 papers on top conferences/journals, 
  including T-PAMI, CVPR, ICCV, NeurIPS, etc. 
  His research interests include deep learning and computer vision.
\end{IEEEbiography}

\vspace{-35pt}
\begin{IEEEbiography}[\addPhoto{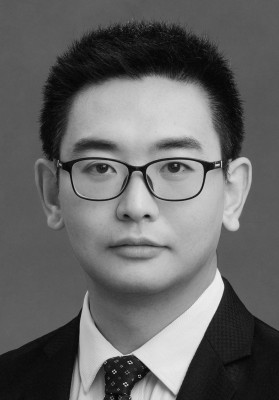}]{Dongwei Ren}
	received two Ph.D. degrees in computer application technology from 
  Harbin Institute of Technology and The Hong Kong Polytechnic University 
  in 2017 and 2018, respectively.
	From 2018 to 2021, he was an Assistant Professor with the 
  College of Intelligence and Computing, Tianjin University.
	He is currently an Associate Professor with the School of Computer 
  Science and Technology, Harbin Institute of Technology. 
	His research interests include computer vision and deep learning.
\end{IEEEbiography}

\vspace{-35pt}
\begin{IEEEbiography}[\addPhoto{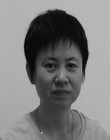}]{Ping Wang}
  received the B.S., M.S., and Ph.D. degrees in computer science from 
  Tianjin University, Tianjin, China, in 1988, 1991, and 1998, respectively.
  She is currently a Professor with the School of Mathematics, 
  Tianjin University.
  Her research interests include image processing and machine learning.
\end{IEEEbiography}

\vspace{-30pt}
\begin{IEEEbiography}[\addPhoto{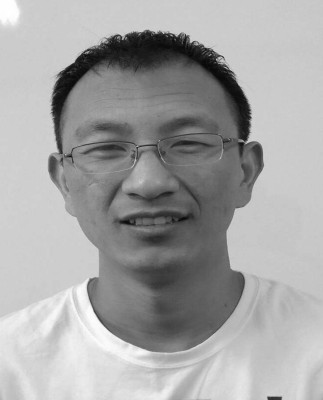}]{Wangmeng Zuo} 
	received the Ph.D. degree from the Harbin Institute of Technology in 2007.
	He is currently a Professor in the School of Computer Science and Technology,
  Harbin Institute of Technology. 
  His research interests include image enhancement and restoration, 
  image and face editing, object detection, visual tracking, 
  and image classification. 
  He has published over 100 papers in top tier journals and conferences. 
  His publications have been cited more than 30,000 times in literature. 
  He is on the editorial boards of IEEE TPAMI and IEEE TIP.
\end{IEEEbiography}

\vspace{-30pt}
\begin{IEEEbiography}[\addPhoto{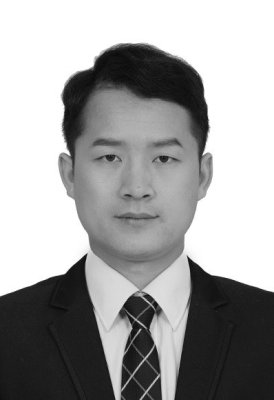}]{Ming-Ming Cheng} 
  received his PhD degree from Tsinghua University in 2012.
  Then he did 2 years research fellow, with Prof. Philip Torr in Oxford.
  He is now a professor at Nankai University, leading the Media Computing Lab.
  His research interests include computer graphics, computer vision, 
  and image processing. 
  He received research awards including 
  National Science Fund for Distinguished Young Scholars
  and ACM China Rising Star Award.
  He is on the editorial boards of IEEE TPAMI and IEEE TIP.
\end{IEEEbiography}

\vfill

\end{document}